\documentclass[10pt,twocolumn,letterpaper]{article}

\usepackage{cvpr}
\usepackage{nopageno}
\usepackage{times}
\usepackage{epsfig}
\usepackage{graphicx}
\usepackage{amsmath}
\usepackage{amsthm}
\usepackage{amssymb}

\usepackage{bbm}
\usepackage{caption}
\usepackage{subcaption}
\usepackage[bookmarks=false]{hyperref}

\usepackage{algorithm}
\usepackage[noend]{algpseudocode}

\usepackage{multirow,bigdelim}
\usepackage{bm}
\usepackage{nicefrac}
\usepackage{verbatim}
\usepackage{booktabs}
\newtheorem{theorem}{Theorem}
\newtheorem{lemma}{Lemma}

\usepackage{listings}

\definecolor{codeblue}{rgb}{0,0,0.6}
\definecolor{codegreen}{rgb}{0,0.6,0}
\definecolor{codegray}{rgb}{0.5,0.5,0.5}
\definecolor{codeturquoise}{rgb}{0.0,0.5,0.5}
\definecolor{backcolour}{rgb}{0.98,0.98,0.98}

\lstdefinestyle{mystyle}{
	backgroundcolor=\color{backcolour},   
	commentstyle=\color{codegreen},
	keywordstyle=\color{codeblue},
	numberstyle=\tiny\color{codegray},
	stringstyle=\color{codeturquoise},
	basicstyle=\footnotesize\ttfamily,
	breakatwhitespace=false,         
	breaklines=true,                 
	captionpos=b,                    
	keepspaces=true,                 
	numbers=left,                    
	numbersep=5pt,                  
	showspaces=false,                
	showstringspaces=false,
	showtabs=false,                  
	tabsize=2
}

\usepackage{makecell}
\usepackage{pifont}
\cvprfinalcopy 

\def\cvprPaperID{****} 

\begin{document}

\thispagestyle{empty}

\title{Learning From Noisy Labels By \\
Regularized Estimation Of Annotator Confusion}


\author{Ryutaro Tanno$^{1}$ \thanks{A part of the work done during internship at Butterfly Network.} \quad \quad Ardavan Saeedi$^{2}$ \quad \quad
Swami Sankaranarayanan$^{2}$ \\ \quad \quad
Daniel C. Alexander$^{1}$ \quad \quad
Nathan Silberman$^{2}$ \\
    $^1$University College London, UK  \quad \quad 
    $^2$Butterfly Network, New York, USA \\
    \vspace{0.in}
    $^1$ {\tt\footnotesize \{r.tanno, d.alexander\}@ucl.ac.uk} \quad \quad 
	$^2${\tt\tt\footnotesize\{asaeedi,swamiviv,nsilberman\}@butterflynetinc.com}\\
\vspace{-0.4in}
}

\maketitle
\begin{abstract} 
\vspace{-3mm}
The predictive performance of supervised learning algorithms depends on the quality of labels. In a typical label collection process, multiple annotators provide subjective noisy estimates of the ``truth" under the influence of their varying skill-levels and biases. Blindly treating these noisy labels as the ground truth limits the accuracy of learning algorithms in the presence of strong disagreement. This problem is critical for applications in domains such as medical imaging where both the annotation cost and inter-observer variability are high. In this work, we present a method for simultaneously learning the individual annotator model and the underlying true label distribution, using only noisy observations. Each annotator is modeled by a confusion matrix that is jointly estimated along with the classifier predictions. We propose to add a regularization term to the loss function that encourages convergence to the true annotator confusion matrix. We provide a theoretical argument as to how the regularization is essential to our approach both for the case of single annotator and multiple annotators. Despite the simplicity of the idea, experiments on image classification tasks with both simulated and real labels show that our method either outperforms or performs on par with the state-of-the-art methods and is capable of estimating the skills of annotators even with a single label available per image.  

\end{abstract}
\vspace{-5mm}
\section{Introduction}
In many practical applications, supervised learning algorithms are trained on noisy labels obtained from multiple annotators of varying skill levels and biases. When there is a substantial amount of disagreement in the labels, conventional training algorithms that treat such labels as the ``truth" lead to models with limited predictive performance. To mitigate such variation, practitioners typically abide by the principle of ``wisdom of crowds" \cite{surowiecki2005wisdom} and aggregate labels by computing the majority vote. However, this approach has limited efficacy in applications where the number of annotations is modest or the tasks are ambiguous. For example,  vision applications in medical image analysis \cite{litjens2017survey} require annotations from clinical experts, which incur high costs and often suffer from high inter-reader variability \cite{watadani2013interobserver,rosenkrantz2013comparison,lazarus2006bi,warfield2004simultaneous}. 



However, if the exact process by which each annotator generates the labels was known, we could correct the annotations accordingly and thus train our model on a cleaner set of data. Furthermore, this additional knowledge of the annotators' skills can be utilized to decide on which examples to be labeled by which annotators \cite{welinder2010online,long2013active,long2015multi}. Therefore, methods that can accurately model the label noise of annotators are useful for improving not only the accuracy of the trained model, but also the quality of labels in the future. 

Previous work proposed various methods for jointly estimating the skills of the annotators and the ground truth (GT) labels. We categorize these methods into two groups: (1) \textit{two-stage} approach and (2) \textit{simultaneous} approach. Methods in the first category perform label aggregation and training of a supervised learning model in two separate steps. The noisy labels $\widetilde{\mathbf{Y}}$ are first aggregated by building a probabilistic model of annotators. The observable variables are the noisy labels $\widetilde{\mathbf{Y}}$, and the latent variables/parameters to be estimated are the annotator skills and GT labels $\mathbf{Y}$. Then, a machine learning model is trained on the pairs of aggregated labels $\mathbf{Y}$ and input examples $\mathbf{X}$ (e.g. images) to perform the task of interest. The initial attempt was made in \cite{dawid1979maximum} in the early 1970s and more recently, numerous lines of research \cite{smyth1995inferring,warfield2004simultaneous,whitehill2009whose,welinder2010multidimensional,rodrigues2013learning} proposed extensions of this work e.g. by estimating the difficulty of each example. However, in all these cases, information about the raw inputs $\mathbf{X}$ is completely neglected in the generative model of noisy labels used in the aggregation step, and this highly limits the quality of estimated true labels in practice.

The \textit{simultaneous} approaches \cite{raykar2009supervised,yan2010modeling,branson2017lean,van2018lean} address this issue by integrating the prediction of the supervised learning model (i.e. distribution $p(\mathbf{Y}|\mathbf{X})$) into the probabilistic model of noisy labels, and have been shown to improve the predictive performance. These methods employ variants of the expectation-maximization (EM) algorithm during training, and require a reasonable number of labels for each example. However, in most real world applications, it is practically prohibitive to collect a large number of labels per example, and this requirement limits their applications. A notable exception is the Model Boostrapped EM (MBEM) algorithm presented in \cite{khetan2017learning} that is capable of learning even with little label redundancy.


In this paper, we propose a more effective alternative to these EM-based approaches for jointly modeling the annotator skills and GT label distribution. Our method separates the annotation noise from true labels by
(1) ensuring high fidelity with the data by minimizing the cross entropy loss and (2) encouraging the estimated annotators to be maximally unreliable by minimizing the trace of the estimated confusion matrices. Our method is also simpler to implement, only requiring an addition of a regularization term to the cross-entropy loss. Furthermore, we provide a theoretical result that such regularization is capable of recovering the annotation noise as long as the average confusion matrix (CM) over annotators is diagonally dominant. 

Experiments on image classification tasks with both simulated and real noisy labels demonstrate that our method, despite being much simpler, leads to better or comparable performance with MBEM \cite{khetan2017learning} and generalized EM \cite{raykar2009supervised,raykar2010learning}, and is capable of recovering CMs even when there is only one label available per example. We simulated a diverse range of annotator types on MNIST and CIFAR10 data sets while we used a ultrasound dataset for cardiac view classification to test the efficacy in a real-world application. We also show importance of modeling individual annotators by comparing against various modern noise-robust methods \cite{reed2014training,sukhbaatar2014training,goldberger2016training,guan2017said}, when the inter-annotator variability is high. 

\vspace{-3mm}
\paragraph{Other Related Works.}
More broadly, our work is related to methods for robust learning in the presence of label noise. There is a large body of literature that do not explicitly model individual annotators unlike our method. 

The effects of label noise are well studied in common classifiers such as SVMs and logistic regression, and robust variants have been proposed \cite{frenay2014classification,natarajan2013learning,bootkrajang2012label}.  More recently, various attempts have been made to train deep neural networks under label noise.  Reed et al. \cite{reed2014training} developed a robust loss to model ``prediction consistency", which was later extended by \cite{tanaka2018joint}. 
In \cite{mnih2012learning} and \cite{sukhbaatar2014training}, label noise was parametrized in the form of a transition matrix and incorporated into neural networks for binary and multi-way classification. A more effective alternative for estimating such transition matrix was proposed in \cite{patrini2017making}, and a method for capturing image dependency of label noise was shown in \cite{Goldberger2017TrainingDN}. We will later compare our model to several of these methods to test the value of modelling individual annotators in gaining robustness to label noise. 

Multiple lines of work have shown that a small portion of clean labels improves robustness. \cite{veit2017learning} proposed to learn from clean labels to correct the labels of noisy examples. \cite{ren2018learning} proposed a method for learning to weigh examples during each training iteration by using the validation loss on clean labels as the meta-objective. \cite{jiang2018mentornet} employs a similar approach, but trains a separate network that proposes weighting. However, curating a set of clean labels of sufficient size is expensive for many applications, and this work focuses on the scenario of learning from purely noisy labels. 



\section{Methods}
 
We assume that a set of images $\{\mathbf{x}_i\}_{i=1}^N$ are assigned with noisy labels $\{\tilde{y}^{(r)}_{i}\}^{r=1,...,R}_{i=1,...,N}$ from multiple annotators where $\tilde{y}^{(r)}_{i}$ denotes the label from annotator $r$ given to example $\mathbf{x}_i$, but no ground truth (GT) labels $\{y_{i}\}_{i=1,...,N}$ are available. In this work, we present a new procedure for multiclass classification problem that can simultaneously estimate the annotator noise and GT label distribution $p(y|\mathbf{x})$ from such noisy set of data $\mathcal{D} = \{\textbf{x}_i, \tilde{y}^{(1)}_{i},...,\tilde{y}^{(R)}_{i}\}_{i=1,...,N}$. The method only requires adding a regularization term, that is the average accuracy of all annotator models, to the cross-entropy loss function. Intuitively, the method biases ours models of each annotator to be as inaccurate as possible while having the model still explain the data. We will show that this is capable of decoupling the annotation noise from the true label distribution, as long as the average labels of the real annotators are ``sufficiently" correct (which we formalize in Sec.~\ref{sec:theorems}). For simplicity, we first describe the method in the \textit{dense label} scenario in which each image has labels from all annotators, and then extend to scenarios with \textit{missing} labels where only a subset of annotators label each image. As we shall see later, the method works even when each image is only labelled by a single annotator. 


\subsection{Noisy Observation Model}
\vspace{-1mm}
\begin{figure*}[ht]
	\center
	\vspace{-5mm}
	\includegraphics[width=0.75\linewidth]{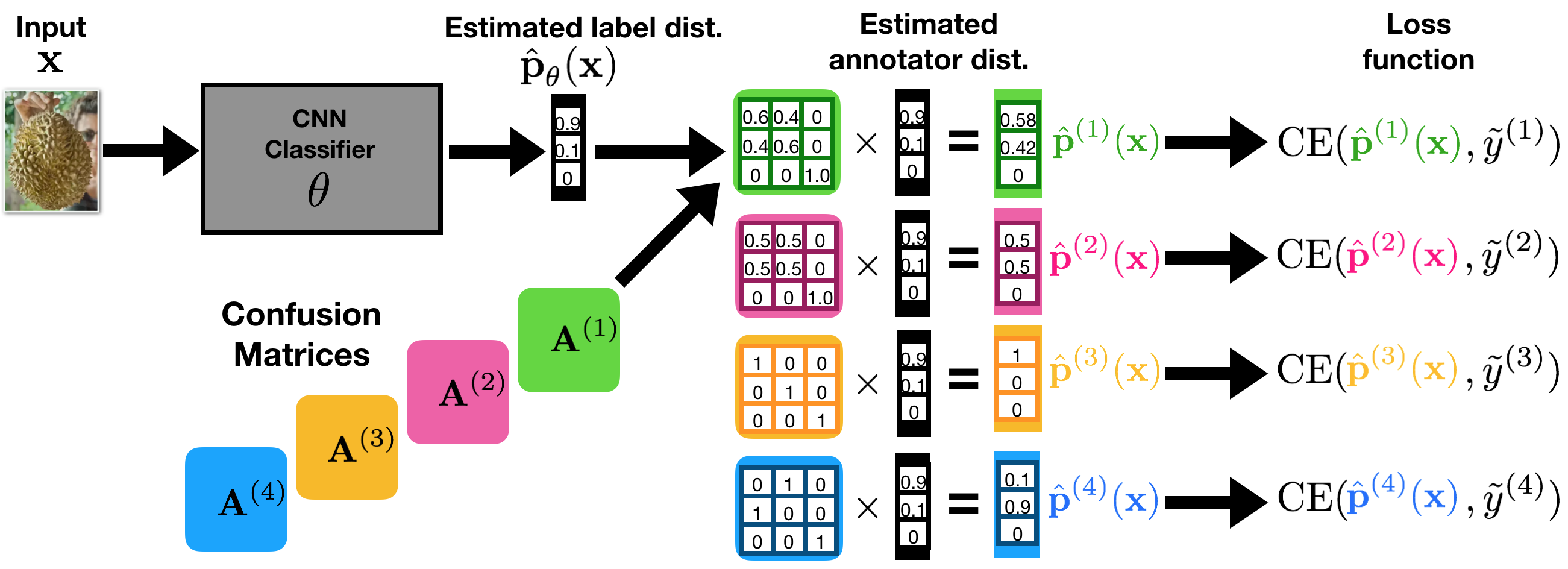}
	\small
	\vspace{-2mm}
	\caption{\small General schematic of the model (eq.~\ref{eq:model_likelihood}) in the presence of $4$ annotators. Given input image $\mathbf{x}$, the classifier parametrised by $\theta$ generates an estimate of the ground truth class probabilities, $\mathbf{p}_{\theta}(\mathbf{x})$. Then, the class probabilities of respective annotators $\mathbf{p}^{(r)}(\mathbf{x}):= \mathbf{A}^{(r)}\mathbf{p}_{\theta}(\mathbf{x})$ for $r\in \{1,2,3,4\}$ are computed. The model parameters $\{\theta, \mathbf{A}^{(1)}, \mathbf{A}^{(2)}, \mathbf{A}^{(3)}, \mathbf{A}^{(4)}\}$ are optimized to minimize the sum of four cross-entropy losses between each estimated annotator distribution $\textbf{p}^{(r)}(\mathbf{x})$ and the noisy labels $\tilde{y}^{(r)}$ observed from each annotator. The probability that each annotator provides accurate labels can be estimated by taking the average diagonal elements of the associated confusion matrix (CM), which we refer to as the ``skill level" of the annotator.}
    \label{fig:schematic}
    \vspace{-4mm}
\end{figure*}

We first describe our probabilistic model of the observed noisy labels from multiple annotators. In particular, we make two key assumptions: (1) annotators are statistically independent, (2) annotation noise is independent of the input image. By assumption (1), the probability of observing noisy labels $\{\tilde{y}^{(1)},...,\tilde{y}^{(R)}\}$ on image  $\textbf{x}$ can be written as:   
\begin{equation}\label{eq:general_likelihood}
p( \tilde{y}^{(1)},...,\tilde{y}^{(R)}|\textbf{x}) =  \prod_{r=1}^{R} \int_{y \in \mathcal{Y}}  p(\tilde{y}^{(r)}|y, \textbf{x}) \cdot p(y|\textbf{x}) dy
\end{equation}
where $p(y|\textbf{x})$ denotes the true label distribution of the image, and $ p(\tilde{y}^{(r)}|y, \textbf{x})$ describes the noise model by which annotator $r$ corrupts the ground truth label $y$. For classification problems, the label $y$ takes a discrete value in $\mathcal{Y}=\{1,...,L\}$.  From assumption (2), the probability that annotator  $r$ corrupts the GT label $y=i$ to $\tilde{y}^{(r)} = j$  is independent of the image $\mathbf{x}$ i.e. $ p(\tilde{y}^{(r)}=j|y=i, \textbf{x}) =  p(\tilde{y}^{(r)}=j|y=i) =: a^{(r)}_{ji} $. Here we refer to the associated $L \times L$ transition matrix $\mathbf{A}^{(r)} = (a^{(r)}_{ji}) $ as the \textit{confusion matrix} (CM) of annotator $r$. The joint probability over the noisy labels is simplified to: 
\begin{equation}\label{eq:model_likelihood}
p( \tilde{y}^{(1)},...,\tilde{y}^{(R)}|\textbf{x}) =  \prod_{r=1}^{R} \sum_{y = 1}^{L}  a^{(r)}_{\tilde{y}^{(r)}, y} \cdot p(y|\textbf{x})
\end{equation}

Fig.~\ref{fig:schematic} provides a schematic of our overall architecture, which models the different constituents in the above joint probability distribution. In particular, the model consists of two components: the \textit{base classifier} which estimates the ground truth class probability vector $\hat{\mathbf{p}}_{\theta}(\mathbf{x})$ whose $i^\text{th}$ element approximates $p(y=i|\textbf{x})$, and the set of the CM estimators $\{\hat{\mathbf{A}}^{(r)}\}_{r=1}^R$ which approximate $\{\mathbf{A}^{(r)}\}_{r=1}^R$. Each product $\hat{\mathbf{p}}^{(r)}(\mathbf{x}) :=\hat{\mathbf{A}}^{(r)}\hat{\mathbf{p}}_{\theta}(\mathbf{x})$ represents the estimated class probability vector of the corresponding annotator. At inference time, we use the most confident class in $\hat{\mathbf{p}}_{\theta}(\mathbf{x})$ as the final classification output. Next, we describe our optimization algorithm for jointly learning the parameters of the base classifier, $\theta$ and the CMs, $\{\hat{\mathbf{A}}^{(r)}\}_{r=1}^R$.

\subsection{Joint Estimation of Confusion and True labels}
\begin{figure*}[h]
	\vspace{-5mm}
    \begin{center}
		\includegraphics[width=0.8\linewidth]{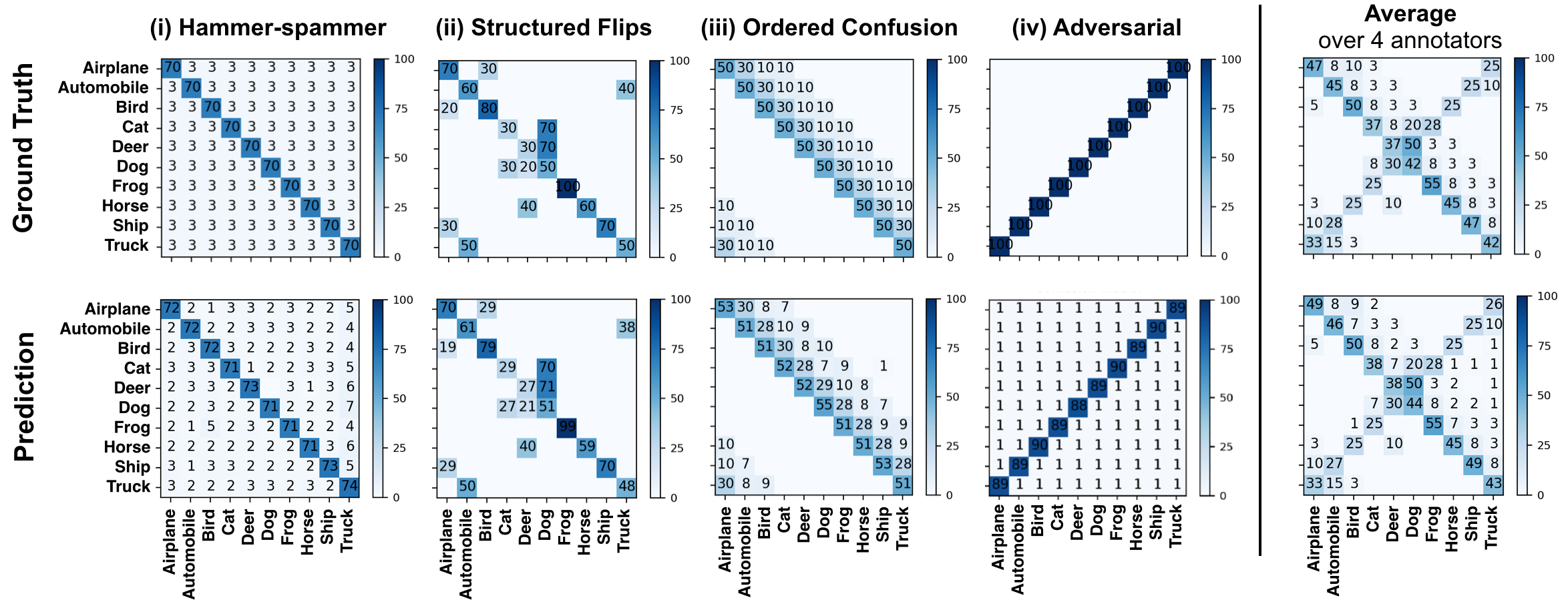}
	\end{center}
	\vspace{-6mm}
	\caption{\small A diverse set of 4 simulated annotators on CIFAR-10. The top row shows the ground truths while the bottom row are the estimation from our method, trained with only one label per image. }
	 \label{fig:recover_cms}
	 \vspace{-3mm}
\end{figure*}

Given training inputs $\mathbf{X} = \{\mathbf{x}_i\}_{i=1}^{N}$ and noisy labels $\widetilde{\mathbf{Y}}^{(r)} = \{ \tilde{y}^{(r)}_{i}\}_{i=1}^{N}$ for $r=1,...,R$, we optimize the parameters $\{\theta, \hat{\mathbf{A}}^{(r)}\}$ by minimizing the negative log-likelihood (NLL), $- \text{log } p(\widetilde{\mathbf{Y}}^{(1)}, ..., \widetilde{\mathbf{Y}}^{(R)}|\mathbf{X})$. From eq.~\ref{eq:model_likelihood}, this optimization objective equates to the sum of cross-entropy losses between the observed labels and the estimated annotator label distributions:
\begin{equation}\label{eq:cross_entropy}
- \text{log } p(\widetilde{\mathbf{Y}}^{(1)}, ..., \widetilde{\mathbf{Y}}^{(R)}|\mathbf{X}) = \sum_{i=1}^{N}\sum_{r=1}^{R}\text{CE}(\textbf{A}^{(r)}\hat{\textbf{p}}_{\theta}(\mathbf{x}_i), \tilde{y}^{(r)}_{i}).
\end{equation}
Minimizing above encourages each annotator-specific prediction $\hat{\mathbf{p}}^{(r)}(\mathbf{x}) :=\hat{\textbf{A}}^{(r)}\hat{\textbf{p}}_{\theta}(\mathbf{x})$ to be as close as possible to the noisy label distribution of the corresponding annotator $\textbf{p}^{(r)}(\mathbf{x})$. However, this loss function alone is not capable of separating the annotation noise from the true label distribution; there are infinite combinations of $\{\hat{\mathbf{A}}^{(r)}\}_{r=1}^R$ and classification model $\hat{\mathbf{p}}_{\theta}$ such that $\hat{\mathbf{p}}^{(r)}$ perfectly matches the annotator's label distribution $\mathbf{p}^{(r)}$ for any input $\mathbf{x}$. 

To formalize this problem, we denote the CM of the estimated true label distribution\footnote{$\textbf{P}_{ji} = \int_{\mathbf{x} \in \mathcal{X}}p(\text{argmax}_{k}[\hat{\textbf{p}}_{\theta}(\mathbf{x})]_{k}=j|y=i)p(\mathbf{x})d\mathbf{x}$} $\hat{\textbf{p}}_{\theta}$ by $\mathbf{P}$. The CM of the estimated annotator's label distribution $\hat{\mathbf{p}}^{(r)}$ is then given by the product $\hat{\textbf{A}}^{(r)}\textbf{P}$. Minimizing the cross-entropy loss (eq.~\ref{eq:cross_entropy}) encourages $\hat{\textbf{A}}^{(r)}\textbf{P}$ to converge to the true CM of the corresponding annotator $\textbf{A}^{(r)}$ i.e. $\hat{\textbf{A}}^{(r)}\textbf{P}\rightarrow \textbf{A}^{(r)}$. However, there are infinitely many solutions pairs $(\hat{\textbf{A}}^{(r)}$, $\textbf{P})$ that satisfy the equality $\hat{\textbf{A}}^{(r)}\textbf{P}= \textbf{A}^{(r)}$. This means that we need to regularize the optimization to encourage convergence to the desired solutions i.e. $\hat{\textbf{A}}^{(r)}\rightarrow\textbf{A}^{(r)}$ and $\textbf{P}\rightarrow \mathbf{I}$. 

To combat this problem, we propose to add the trace of the estimated CMs to the loss in eq.~\ref{eq:cross_entropy}. Extending to the ``missing labels" regime in which only a subset of annotators label each example, we derive the combined loss:
\begin{equation}\label{eq:objective_sparse}
\sum_{i=1}^{N}\sum_{r=1}^{R}\mathbbm{1}(\tilde{y}_{i}^{(r)} \in \mathcal{S}(\mathbf{x}_i))\cdot \text{CE}(\hat{\mathbf{A}}^{(r)}\hat{\textbf{p}}_{\theta}(\mathbf{x}_i), \tilde{y}^{(r)}_{i})  + \lambda \sum_{r=1}^{R}\text{tr}(\hat{\textbf{A}}^{(r)})
\end{equation}
where $\mathcal{S}(\mathbf{x})$ denotes the set of all labels available for image $\mathbf{x}$, and $\text{tr}(\mathbf{A})$ denotes the trace of matrix $\mathbf{A}$. We simply perform gradient descent on this loss to learn $\{\theta, \hat{\mathbf{A}}^{(1)}, ..., \hat{\mathbf{A}}^{(R)}\}$.

Numerous previous work have considered the same observation model, but proposed various optimization schemes. The original work \cite{raykar2009supervised,raykar2010learning} employed the generalized EM algorithm to estimate $\{\theta, \hat{\mathbf{A}}^{(1)}, ..., \hat{\mathbf{A}}^{(R)}\}$, and more recent work  \cite{branson2017lean,van2018lean} employed variants of hard-EM to optimize the same model. Khetan et al.,\cite{khetan2017learning} proposed a method called model-bootstrapped EM (MBEM) in which the predictions of the base neural network classifier are used in the M-step update of CMs to learn from singly labelled data, which was not viable with the prior work. However, in all of the above EM-based methods, each M-step for the parameters of NN is not available in closed form and thus performed via gradient descent. This means that every M-step requires a training of the CNN classifier, rendering each iteration of EM expensive. A naive solution to this is to perform only few iterations of gradient descent in each E-step, however, this could limit the performance if sufficient convergence is not achieved. Our approach directly maximizes the likelihood with the trace regularizer and does not suffer from these issues. In Sec.~4, we show empirically this approach leads to an improvement both in terms of accuracy and convergence rate over the previous methods on noisy labels with high inter-annotator variability. 

\subsection{Motivation for Trace Regularization}\label{sec:theorems}
Here we intend to motivate the addition of the trace regularizer in eq.~\ref{eq:objective_sparse}. In the last section, we saw that minimizing cross-entropy loss alone encourages $\hat{\textbf{A}}^{(r)}\textbf{P}\rightarrow \textbf{A}^{(r)}$. Therefore, if we could devise a regularizer which, when minimized, uniquely ensures the convergence $\hat{\textbf{A}}^{(r)}\rightarrow\textbf{A}^{(r)}$, then this would make $\textbf{P}$ tend to the identity matrix, implying that the base model fully captures the true label distribution i.e. $\text{argmax}_{k}[\hat{\textbf{p}(\mathbf{x})}_{\theta}]_{k} =  y\,\, \forall \mathbf{x}$. We describe below the trace regularizer is indeed a such regularizer when both $\hat{\mathbf{A}}^{(r)}$ and $\mathbf{A}^{(r)}$ satisfy some conditions. We first show this result assuming that there is a single annotator, and then extend to the scenario with multiple annotators.  

\vspace{1mm}
\begin{lemma}[Single Annotator]\label{lemma:1}
Let $\textbf{P}$ be the CM of the estimated true labels $\hat{\textbf{p}}_{\theta}$ and $\hat{\textbf{A}}$ be the estimated CM of the annotator. If the model matches the noisy label distribution of the annotator i.e. $\hat{\textbf{A}}\textbf{P}=\textbf{A} $, and both $\hat{\textbf{A}}$ and $\textbf{A}$ are diagonally dominant ($a_{ii} > a_{ij}$, $\hat{a}_{ii} > \hat{a}_{ij}$) for all $i \neq j$, then $\hat{\textbf{A}}$ with the minimal trace uniquely coincides with the true $\textbf{A}$. 
\end{lemma}

\begin{proof}
We show that each diagonal element in the true CM $\textbf{A}$ forms a lower bound to the corresponding element in its estimation. 
\vspace{-1mm}
\begin{equation}\label{eq:side1}
    a_{ii} = \sum_{j} \hat{a}_{ij}p_{ji} \leq \sum_{j}\hat{a}_{ii}p_{ji} = \hat{a_{ii}}(\sum_{j}p_{ji}) = \hat{a}_{ii} 
\vspace{-2mm}
\end{equation}
for all $i\in\{1,...,L\}$. It therefore follows that $\text{tr}(\textbf{A})\leq \text{tr}(\hat{\textbf{A}})$. We now show that the equality $\hat{\textbf{A}} = \textbf{A}$ is uniquely achieved when the trace is the smallest i.e. $\text{tr}(\textbf{A})= \text{tr}(\hat{\textbf{A}})	\Rightarrow \textbf{A} = \hat{\textbf{A}}$. From \eqref{eq:side1}, if the trace of $\textbf{A}$ and $\hat{\textbf{A}}$ are the same, we see that their diagonal elements also match i.e. $ a_{ii} = \hat{a}_{ii} \forall i\in\{1,...,L\}$. Now, the non-negativity of all elements in CMs $\textbf{P}$ and $\hat{\textbf{A}}$, and the equality $a_{ii} = \sum_{j} \hat{a}_{ij}p_{ji}$ imply that $p_{ji} = \mathbbm{1}[i=j]$ i.e. $\textbf{P}$ is the identity matrix.



\end{proof}
\vspace{-2mm}
We note that the above result was also mentioned in \cite{sukhbaatar2014training} in a more general context of label noise modelling (that neglects annotator information). Here we further augment their proof by showing the uniqueness of solutions (i.e.  $\text{tr}(\textbf{A})= \text{tr}(\hat{\textbf{A}})	\Rightarrow \textbf{A} = \hat{\textbf{A}}$). In addition, the trace regularization was never used in practice in \cite{sukhbaatar2014training} --- for implementation reason, the Frobenius norm was used in all their experiments.  We now extend this to the multiple annotator regime. We will show later that minimizing the mean trace of all annotators indeed enhances the estimation quality of both CM and true label distributions, particularly in the presence of high annotator disagreement. 

\begin{theorem}[Multiple Annotators]\label{theorem:main}
Let $\hat{\textbf{A}}^{(r)}$ be the estimated CM of annotator $r$. If $\hat{\textbf{A}}^{(r)}\textbf{P}=\textbf{A}^{(r)}$ for $r=1,...,R$, and the average true and estimated CMs $\textbf{A}^{*}:=R^{-1}\sum_{r=1}^R\textbf{A}^{(r)}$ and $\hat{\textbf{A}}^{*}:=R^{-1}\sum_{r=1}^R\hat{\textbf{A}}^{(r)}$ are diagonally dominant, then  $\textbf{A}^{(1)}, ..., \textbf{A}^{(R)} = \text{argmin }_{\hat{\textbf{A}}^{(1)}, ..., \hat{\textbf{A}}^{(R)}}\Big{[}\text{tr}(\hat{\textbf{A}}^{*})\Big{]}$ and such solutions are unique. In other words, when the trace of the mean CM is minimized, the estimation of respective annotator's CMs match the true values. 
\end{theorem}

\begin{proof}
As the average CMs $\textbf{A}^{*}$ and $\hat{\textbf{A}}^{*}$ are diagonally dominant and we have $\textbf{A}^{*}=\hat{\textbf{A}}^{*}\textbf{P}$, Lemma~\ref{lemma:1} yields that $\text{tr}(\textbf{A}^{*}) \leq \text{tr}(\hat{\textbf{A}}^{*}) $ with equality if and only if $\textbf{A}^{*} = \hat{\textbf{A}}^{*}$. Therefore, when the trace of the average CM of annotators is minimized i.e. $ \text{tr}(\hat{\textbf{A}}^{*})=\text{tr}(\textbf{A}^{*}) $, the estimated CM of the true label distribution $\textbf{P}$ reduces to identity, giving $\hat{\textbf{A}}^{(r)}= \textbf{A}^{(r)}$ for all $r \in \{1,...,R\}$. 
\end{proof}
\vspace{-2mm}
The above result shows that if each estimated annotator's distribution $\hat{\textbf{A}}^{(r)}\hat{\textbf{p}}_{\theta}(\mathbf{x})$ is very close to the true noisy distribution $\textbf{p}^{(r)}(\mathbf{x})$ (which is encouraged by minimizing the cross-entropy loss), and on average for each class $c$, the number of correctly labelled examples exceeds the number of examples of every other class $c'$ that are mislabelled as $c$ (the mean CM is diagonally dominant), then minimizing its trace will drive the estimates of CMs towards the true values. To encourage $\{\hat{\mathbf{A}}^{(1)}, ..., \hat{\mathbf{A}}^{(R)}\}$ to be also diagonally dominant, we initialize them with identity matrices. Intuitively, the combination of the trace term and cross-entropy separates the true distribution from the annotation noise by finding the maximal amount of confusion which can explain the noisy observations well. 


\section{Experiments}

We now aim to verify the proposed method on various image recognition tasks. Particularly, we demonstrate (1) advantage of our simpler optimization scheme compared to EM-based approaches (Sec.~\ref{sec:comparison_with_em}), (2) importance of modeling multiple annotators (Sec.~\ref{sec:other_baselines}) and (3) the applicability of the model in a challenging real world application (Sec.~\ref{sec:us_experiments}). We address the first two questions by testing the proposed method on MNIST and CIFAR-10 datasets with a diverse set of simulated annotators. To answer the final question, we evaluate our approach on the task of cardiac view classification using ultrasound images where the labels are noisy and sparse, and are acquired from multiple annotators of varying levels of expertise. 



\subsection{Set-Up}
We focus on a regime in which models have only access to noisy labels from multiple annotators. For MNIST and CIFAR-10 data sets, we simulate noisy labels from a range of annotators with different skill levels and biases.

\paragraph{MNIST Experiments.} We consider two different models of annotator types: (i) \textit{pairwise-flipper}: each annotator is correct with probability $p$ or flips the label of each class to another label (the flipping target is chosen uniformly at random for each class), (ii) \textit{hammer-spammer}: each annotator is always correct with probability $p$ or otherwise chooses labels uniformly at random \cite{khetan2017learning}. For each annotator type and skill level $p$, we create a group of $5$ annotators by generating CMs from the associated distribution (illustration of CMs are given in the supplementary material). Given the GT labels, we generate noisy labels as defined by the CM per annotator. These noisy labels are used during training.


\paragraph{CIFAR-10 Experiments.} We consider a diverse group of $4$ annotators with different patterns of CMs as shown in Fig.~\ref{fig:recover_cms}: (i) is a ``hammer-spammer" as defined above, (ii) tends to mix up semantically similar categories of images e.g. cats and dogs, and automobiles and trucks, (iii) is likely to confuse ``neighbouring" classes and (iv) is an adversarial annotator who has a wrong association of class names to object categories. On average, labels generated by these annotators are correct only $45\%$ of the time.

In synthetic experiments, we assume that equal number of labels are generated by each annotator on average. We also note that all models are trained on noisy labels and do not have access to the ground truth. Unless otherwise stated, we hold out $10\%$ of training images as a validation set, on which the best performing model is selected. We also perform no data augmentation during training. Full details of training and model architectures are provided in the supplementary material. In Sec.~\ref{sec:comparison_with_em} and Sec.~\ref{sec:other_baselines} below, we compare our model against two separate sets of baselines to address different questions. 

\begin{figure}[ht]
	\center
	\begin{subfigure}[]{1.0\linewidth}
		\includegraphics[width=\linewidth]{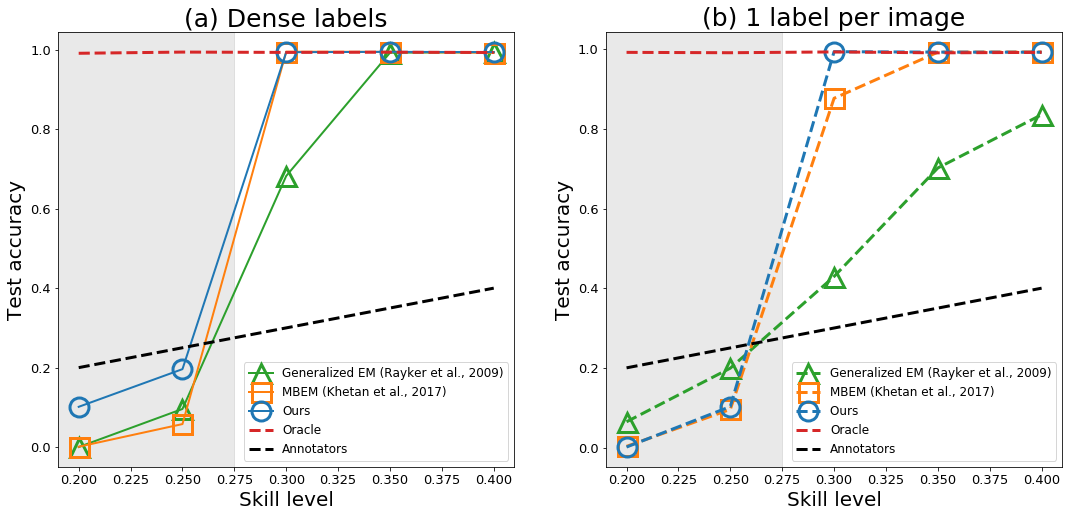}
	\end{subfigure}
	\hspace{10mm}
	\begin{subfigure}[]{1.0\linewidth}
		\includegraphics[width=\linewidth]{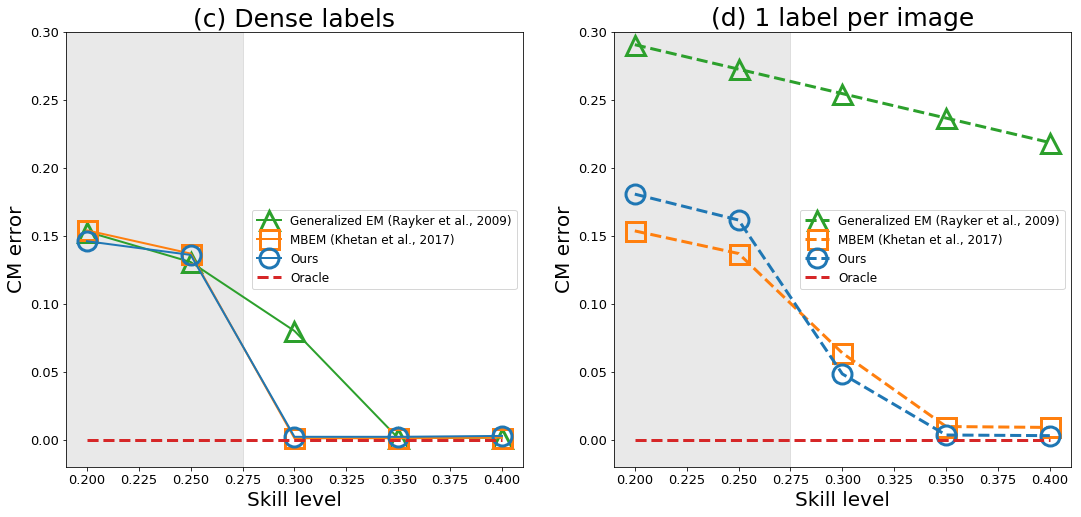}
	\end{subfigure}
	\caption{\small Comparison between our method, generalized EM, MBEM trained on noisy labels on MNIST from ``pairwise flippers" for a range of mean skill level $p$. (a), (b) show classification accuracy in two cases, one where all annotators label each example and the other where only one label is available per example. (c), (d) quantify the CM recovery error as the annotator-wise average of the normalized Frobenius norm between each ground truth CM and its estimate. The shaded areas represent the cases where the average CM over the annotators are not diagonally dominant. }
	\label{fig:pairwise_flips}
	\vspace{-2mm}
\end{figure}
\begin{figure}[t]
	\center
	\vspace{-2mm}
	\includegraphics[width=\linewidth]{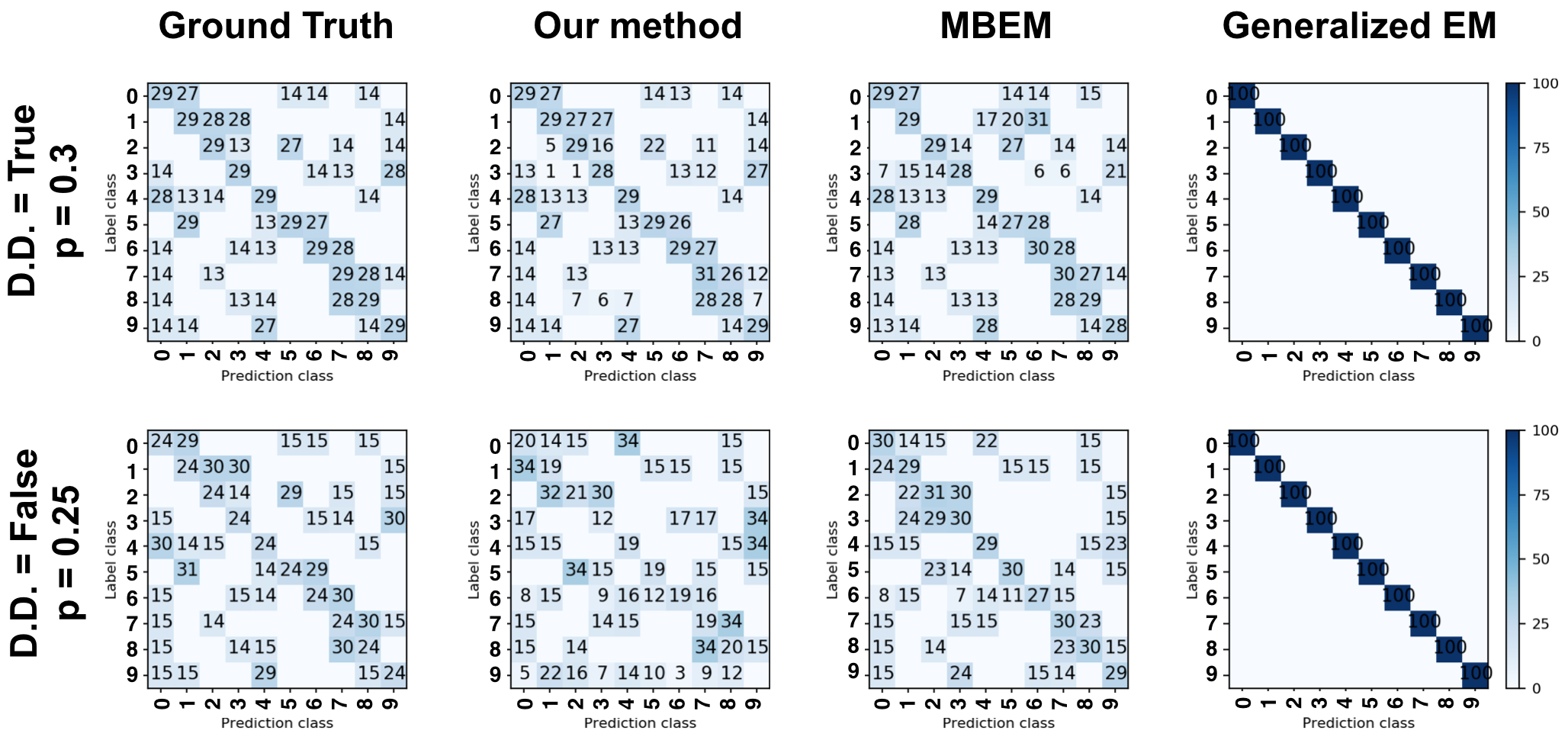}
	\vspace{-7mm}
	\caption{\small Visualization of the mean CM estimates when the diagonal dominance (D.D.) holds (mean skill level, $p=0.3$) and does not hold ($p=0.25$). In all cases, only one label is provided per image. The numbers are rounded to nearest integers. Here the respective models are trained on the noisy labels from $5$ ``pairwise flippers". Note that when each image receives only 1 label, the generalised EM \cite{raykar2009supervised} completely fails to recover the CM due to the failure of M-step for updating the confusion matrices (see Algorithm. 2 in the supplementary material). }
	\label{fig:diagonal_dominance_violation}
	\vspace{-2mm}
\end{figure}

\subsection{Comparing with EM-based Approaches} \label{sec:comparison_with_em}
This section examines the ability of our method in learning the CMs of annotators and the GT label distribution on MNIST and CIFAR-10. In particular, we compare against two prior methods: (1) generalized EM \cite{raykar2010learning}, the first method for end-to-end training of the CM model in the presence of multiple annotators, and (2) Model Bootstrapped EM (MBEM) \cite{khetan2017learning}, the present state-of-the-art method. We analyze the performance in two cases, one in which all labels from $5$ annotators are available for each image (``dense labels"), and another where only one randomly selected annotator labels each example (``1 label per image"). We quantify the error of CM estimation by the average Frobenius norm between each CM and its estimate over the annotators, and this metric is normalized to be in the range $[0,1]$ by dividing by the number of classes $L$ i.e. $R^{-1}L^{-1}\sum_{r}\sum_{i,j}||a^{(r)}_{ij}-\hat{a}^{(r)}_{ij}||^2$.

\paragraph{Performance Comparison.} Fig.~\ref{fig:pairwise_flips} compares the classification accuracy and the error of CM estimation on MNIST for a range of mean skill-levels $p$ where labels are generated by a group of 5 ``pairwise-flippers". The ``oracle" model is the idealistic scenario where CMs of the annotators are a priori known to the model while ``annotators" indicate the average labeling accuracy of each annotator group.

Fig.~\ref{fig:pairwise_flips} shows a strong correlation between the classification accuracy and the error of CM estimation. We observe our model displays consistently better or comparable performance in terms of both classification accuracy and estimation of CMs with dense labels (Fig.~\ref{fig:pairwise_flips}(a) and (c)). When each example receives only one label from one of the annotators, we observe the same trend as long as the mean CMs are diagonally dominant (Fig.~\ref{fig:pairwise_flips}(b,d)). We also observe that when the diagonal dominance holds, all three methods perform better than the annotators. On the other hand, when the diagonal dominance does not hold (see the grey regions), all models undergo a steep drop in classification accuracy due to the inability to estimate CMs accurately as reflected in Fig.~\ref{fig:pairwise_flips}(c,d), which is consistent with Theorem.~\ref{theorem:main}. Fig.~\ref{fig:diagonal_dominance_violation} also visualizes the average of the estimated CMs at this break point. We also note that with only one label per image, the generalized EM algorithm  \cite{raykar2009supervised,raykar2010learning} is not capable of recovering CMs at all and predict identity matrices (Fig.~\ref{fig:diagonal_dominance_violation}), which renders the model equivalent to a vanilla classifier directly trained on noisy labels. A similar set of results in the ``spammer-hammer" case are also available in the supplementary materials.

On CIFAR-10 dataset, Tab.~\ref{tab:cifar10} shows that our method outperforms MBEM and the generalized EM in terms of both classification accuracy and CM estimation by a large margin. In addition, the standard deviations of these metrics are generally smaller for our method than for the baselines. Fig.~\ref{fig:recover_cms} illustrates that our method can estimate CMs of the 4 very different annotators even when each image receives only one label. Interestingly, Tab.~\ref{tab:cifar10} shows that even removing the trace norm can achieve reasonably high classification accuracy and low CM estimation error. We believe this is because of the unexplained robustness of a deep CNN to label noise. Nevertheless, adding the trace norm improves the performance, and we also observe on MNIST that such improvement is pronounced in the presence of larger noise (see supplementary materials).




\begin{table}
    \begin{subtable}[t]{0.9\linewidth}
    \footnotesize
        \begin{tabular}{lcc}
            \hline
                \toprule
                Method    & Accuracy  & CM error \\
                \midrule
                Our method                                  & $\bm{81.23 \pm 0.21}$ & $\bm{0.72 \pm  0.01}$ \\
                Our method (no trace norm)                  & $80.29 \pm 0.65$ & $1.37 \pm  0.12$ \\
        
                MBEM \cite{khetan2017learning}              & $73.33 \pm 0.46$ & $2.53 \pm 0.24$  \\
                generalized EM \cite{raykar2009supervised}  & $70.49 \pm 0.23$ & $6.13 \pm 0.28$  \\
                \midrule
                Single CM \cite{sukhbaatar2014training}  & $68.82 \pm 2.27$  & -   \\
                Weighted Doctor Net \cite{guan2017said}     & $60.11 \pm 1.80$   &  -  \\
                Soft-bootstrap \cite{reed2014training}      & $54.73 \pm 1.33$ & -    \\
                Vanilla CNN \cite{reed2014training}         & $52.33 \pm 0.31$ & -    \\
        \hline
        \end{tabular}
        \caption{Dense labels}
        \end{subtable}\hfill
    \vspace{5mm}
    \begin{subtable}[t]{0.9\linewidth}
    \footnotesize
        \begin{tabular}{lcc}
        \hline
            \toprule
            Method    & Accuracy  & CM error \\
            \midrule
            Our method                                  & $\bm{77.65\pm0.31}$ & $\bm{1.22\pm0.01}$\\
            Our method (no trace norm)                  & $76.31\pm0.49$ & $1.46 \pm 0.27$ \\
    
            MBEM \cite{khetan2017learning}              & $55.97 \pm 1.23$ & $4.58\pm0.64$ \\
            generalized EM \cite{raykar2009supervised}  & $53.38 \pm 0.71$ & $4.47\pm0.64$  \\
            \midrule
            Single CM \cite{sukhbaatar2014training}  & $59.91 \pm 0.98$  & -   \\
            Weighted Doctor Net \cite{guan2017said}     & $57.98\pm0.14$ &  -  \\
            Soft-bootstrap \cite{reed2014training}      & $42.91\pm1.08$ & -    \\
            Vanilla CNN \cite{reed2014training}         & $36.04\pm1.04 $ & -    \\
        \hline
        \end{tabular}
        \caption{1 label per image}
    \end{subtable}
\caption{\small Mean classification accuracy and CM estimation errors ($\times 10^{-2}$) on CIFAR-10 with dense labels. Average annotator accuracy is $45\%$. Standard deviations are computed based on 3 runs with varied weight initialization.}
\label{tab:cifar10}
\vspace{-7mm}
\end{table}

\paragraph{Sensitivity to Hyper-parameters.} 
We next study the robustness of our method against the generalized EM and MBEM to the specification of hyper-parameters. We used the group of five pairwise-flippers with the mean skill level $p=0.35$ to generate noisy labels on MNIST data set. For our model, we compare the effects of the scaling $\lambda$ of the trace-norm in eq.~\ref{eq:objective_sparse} on the trajectory of classification accuracy on the validation set and the quality of CM estimation. For the baselines, we experiment by varying the number of EM steps (denoted by $T$) and the number of stochastic gradient descent for each E-step (denoted by $G$) while fixing the total number of training iterations at $100,000$. We observed our model presents robustness to different values of $\lambda$ as long as the trace-norm loss is not larger than the cross-entropy loss (where the estimated CMs will start to diffuse too much), and Fig.~\ref{fig:robustness_to_hyperparams} shows the stability of the validation curves for $\lambda \in \{0.1, 0.01, 0.001\}$. Both the MBEM and generalized EM show evident dependence on the values of $T$ and $G$ and by and large display slower convergence than our method. We also observe that if too few gradient descents are performed ($G = 1000$) during each E-step, the model converges to a lower accuracy in both classification and CM estimation.

\begin{figure}[ht]
	\center
	\vspace{-2mm}
	\includegraphics[width=0.97\linewidth]{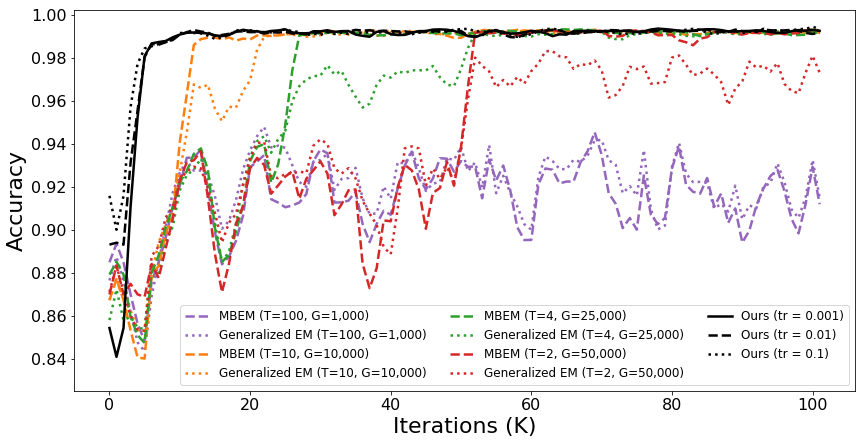}
	\vspace{-2mm}
	\caption{\small Curves of validation accuracy during training of our method, generalized EM and MBEM for a range of hyper-parameters. For our method, the scaling of the trace regularizer is varied in $[0.001, 0.01, 0.1]$. while, for EM and MBEM, we vary the number of EM steps ($T$), and the number of gradient descent steps per E-step ($G$) while fixing the total number of training iterations at $100,000$.}
	\label{fig:robustness_to_hyperparams}
	\vspace{-4mm}
\end{figure}

\begin{figure}[ht]
    \vspace{-1mm}
	\begin{center}
		\includegraphics[width=\linewidth]{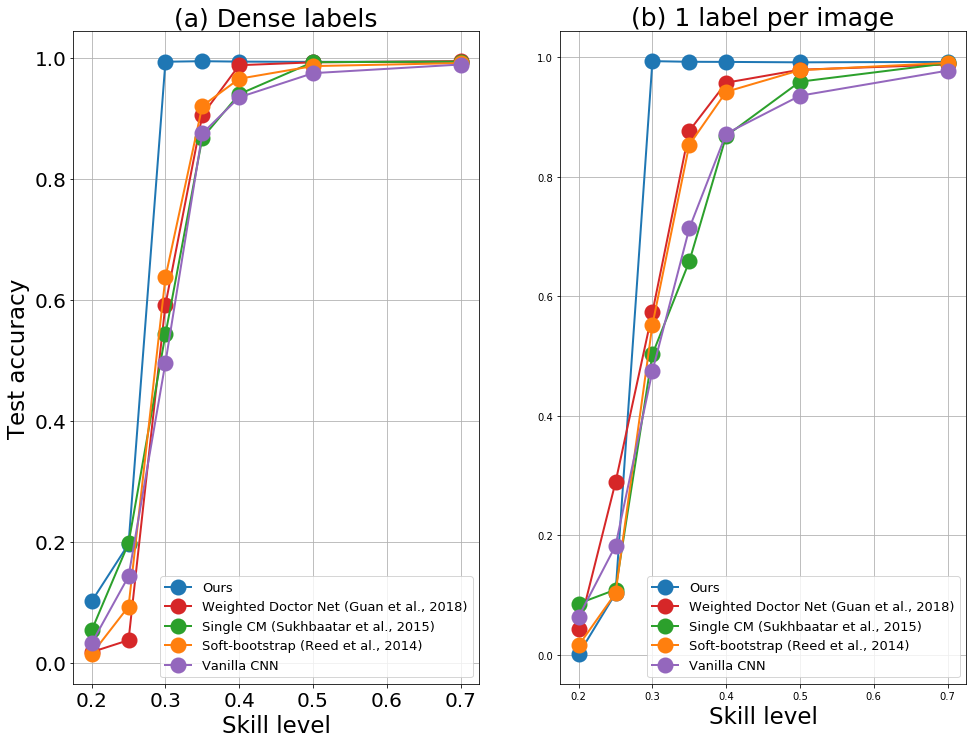}
	\end{center}
 	\vspace{-7mm}

	\caption{\small Classification accuracy on MNIST of different noise-robust models as a function of the mean annotator skill level $p$ in two cases. Here, for each mean skill-level $p$, a group of $5$ ``pairwise flippers" is formed and used to generate labels. (a). each example receives labels from all the annotators. (b). each example is labelled by only 1 randomly selected annotator. }
	\label{fig:accuracy_vs_noise_levels}
	\vspace{-4mm}

\end{figure}

\label{sec:us_experiments}
\begin{figure*}[h]
	\vspace{-8mm}
	\hspace{0mm}
	\center
	\begin{subfigure}[]{0.6\linewidth}
 		\vspace{-2mm}
		\includegraphics[width=\linewidth]{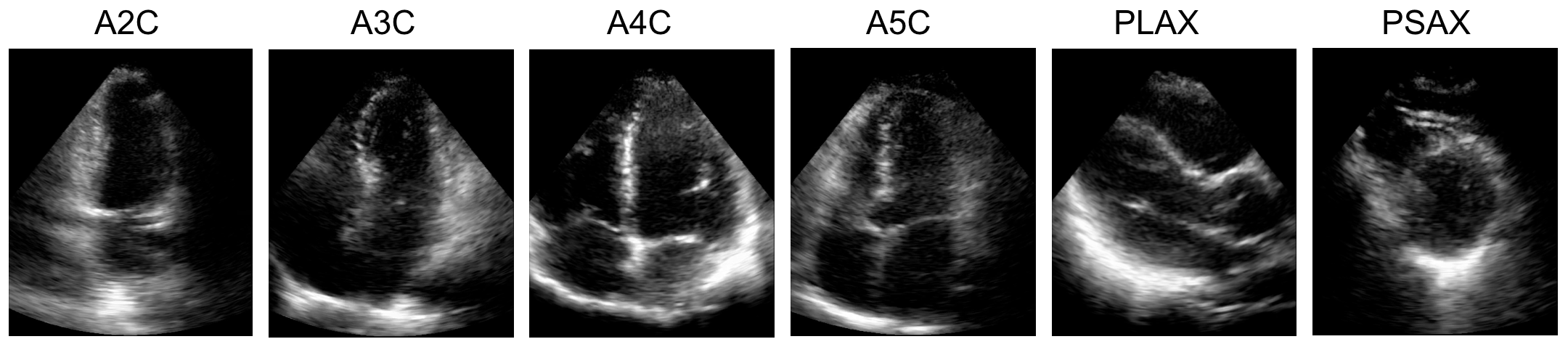}
 		\vspace{-6mm}
		\caption{Different classes of cardiac views}
	\end{subfigure}
	\hspace{30mm}
	\hfill
	\begin{subfigure}[]{0.22\linewidth}
		\includegraphics[width=0.9\linewidth]{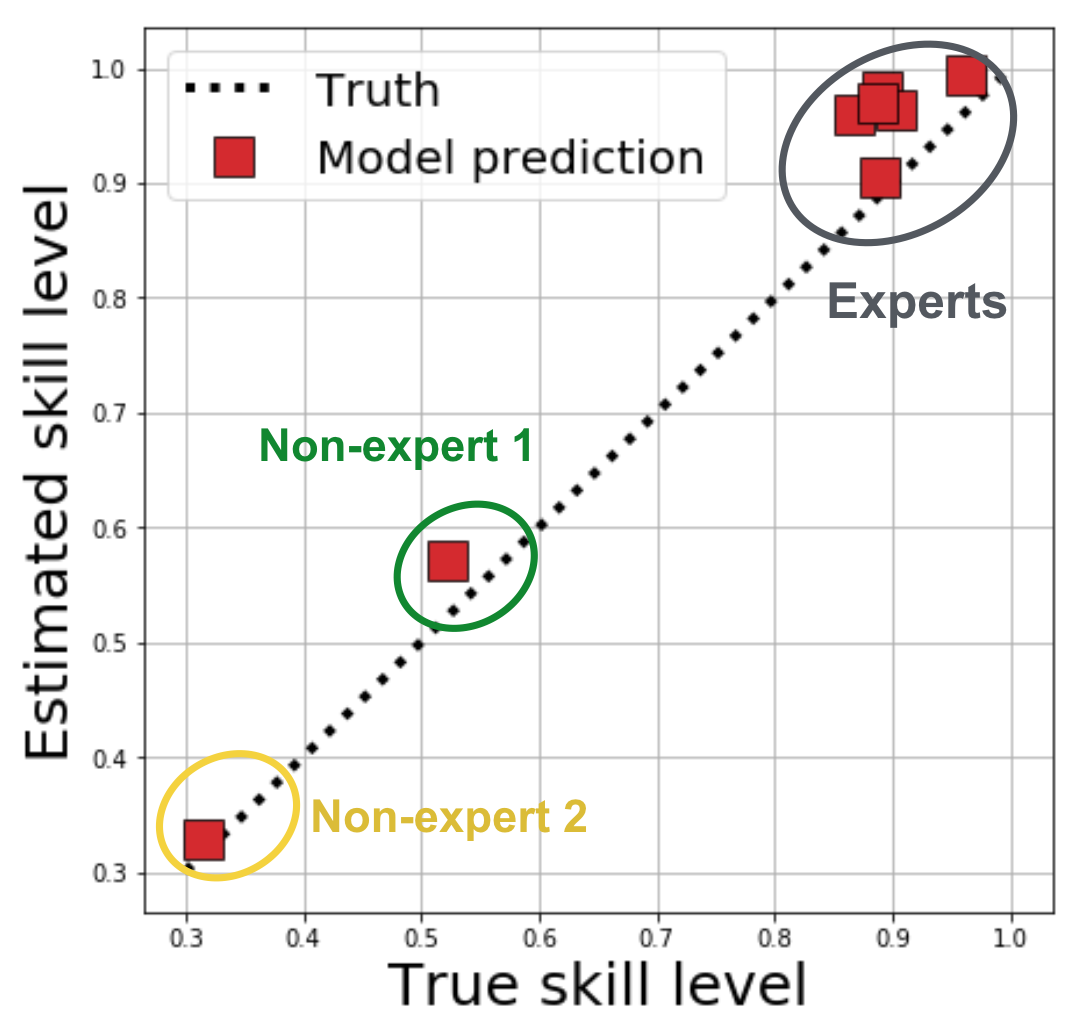}
		\vspace{-2mm}
		\caption{Skill estimation}
	\end{subfigure}
	\hspace{0mm}
	\begin{subfigure}[]{0.35\linewidth}
		\includegraphics[width=\linewidth]{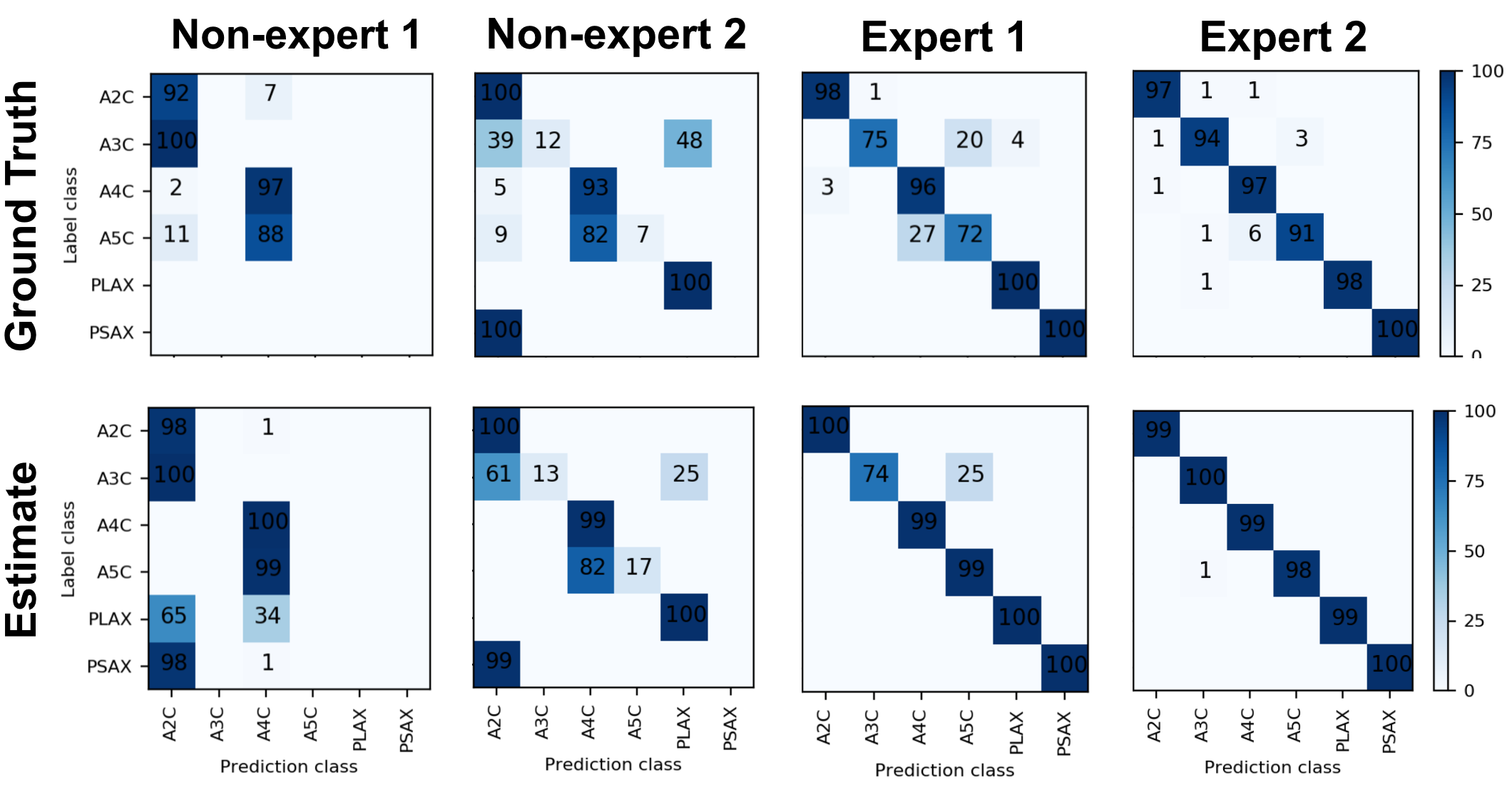}
		\vspace{-5mm}
		 \caption{Learned CMs}
	\end{subfigure}
	\hspace{0mm}
	\begin{subfigure}[]{0.35\linewidth}
	    \footnotesize
	    \begin{tabular}{|c|c c|} 
             \hline
             Method & Accuracy & CM error \\ 
             \hline
             Ours & $\bf{75.57\pm 0.16}$ &  $\bf{11.48 \pm 0.48}$ \\ 
             w/o trace norm & $70.99 \pm 3.31$ & $15.22 \pm0.94$  \\ 
             MBEM \cite{khetan2017learning} & $73.91\pm 0.11$ & $12.18\pm0.29$ \\ 
             WDN \cite{guan2017said} & $59.15 \pm 1.60$ & - \\ 
             Single-CM \cite{sukhbaatar2014training} & $74.38\pm0.29$ & - \\ 
             Soft-bootstrap \cite{reed2014training} & $72.99\pm 0.17$ & - \\ 
             Vanilla CNN & $70.95\pm 0.44$ & -\\ 
             \hline
        \end{tabular}
        \caption{Performance comparison}
	\end{subfigure}
	\vspace{-3mm}
	\caption{\small Results on the cardiac view classification dataset: (a) illustrates examples of different cardiac view images. (b) plots the estimated skill level of each annotator (average of the diagonal elements of its estimated CM)  against the ground truth (c) compares the estimated CMs of the two least skilled and two most skilled annotators according the GT labels (d) summarizes the classification accuracy and error of CM estimation for different methods. }
	\label{fig:us_experiments}
	\vspace{-4mm}
\end{figure*}

\subsection{Value of Modelling Individual Annotators}

\label{sec:other_baselines}
Now, we compare the performance of our method against the prior work that aim to improve robustness to noisy labels without explicitly modelling the individual annotators. The first baseline is the vanilla classifier trained on the majority vote labels. We also compare against the noise robust approaches proposed in \cite{reed2014training} and \cite{sukhbaatar2014training}. Reed \textit{et al.} \cite{reed2014training} adds to the cross-entropy loss a label consistency term based on the negative entropy of the softmax outputs, and we used the default hyper-parameter $\beta = 0.95$ for comparison. Sukhbaatar \textit{et al.} \cite{sukhbaatar2014training} explicitly accounts for the label noise with a single CM, but does not model individual annotators. We add the trace-norm of the same scaling used in our method ($\lambda=0.01$) to the loss function for training. We also include Weighted Doctor Net architecture (WDN) \cite{guan2017said} in the comparison, a recent method that models the annotators individually and then learns averaging weights for combining them. It should be noted that this model considers a different observation model of the labels and does not explicitly model the true label distribution. When we have access to multiple labels per example, with the exception of WDN, we aggregated the labels by computing the majority vote and trained all models. This is because we observed a consistent improvement on validation accuracy (thus poses a tougher challenge against our method) and this would be a more realistic utilization of such data set. For both MNIST and CIFAR-10 experiments, we test on the same set of simulated labels as used in Sec.~\ref{sec:comparison_with_em}.

Fig.~\ref{fig:accuracy_vs_noise_levels} shows better or comparable classification accuracy than all the baselines when the diagonal dominance of the mean CM holds. In particular, our methods show significant improvement when the mean skill level of the annotators are relatively low (e.g. $p=0.3$ and $0.35$). The results are pronounced in the case with only one label available per image for which the baseline methods undergo a steep drop in accuracy (see Fig.~\ref{fig:accuracy_vs_noise_levels}(b)). Results in the ``spammer-hammer" case are available in the supplementary material. Similarly on CIFAR-10 data set, Tab.~\ref{tab:cifar10} shows that our method improves the classification accuracy upon the baselines. Such improvement is pronounced in the case of sparse labels. On the other hand, a vanilla CNN with only L2 weight decay overfits to the training data very quickly in the presence of such high noise.


\vspace{-2mm}
\subsection{Experiments on Cardiac View Classification}
\vspace{-2mm}

Lastly, we illustrate the results of our approach for a real data set with sparse and noisy labels from the medical domain. This data set consists of images of the cardiac region in different views, acquired using a hand-held ultrasound probe. The task is to classify a given ultrasound image into one of six different view classes (see Fig.~\ref{fig:us_experiments}(a)). The process of obtaining a cardiac view label is crucial for guiding the user to the correct locations of measurements, and affects the quality of the downstream cardiac tasks. 

A committee of sonographers (with varying levels of experience) were tasked with providing the cardiac view labels to a large volume of ultra-sound images, and each example is only labelled by a subset of them. To acquire ground truth in this setting, we chose those samples where the three most experienced sonographers agreed on a given label. The resulting data set consists of noisy labels provided by the remaining less experienced $6$ sonographers for a total of $240,000$ training images and $22,000$ validation images. In addition, we also acquired labels from two non-expert users and included in the training data. 

We estimated the skill-level of each annotator by computing the average value of the diagonal elements in the corresponding learned CM, and Fig.~\ref{fig:us_experiments}(b) shows that the group of experts can be separated from the two non-experts with varying levels of experinces (one is less competent than the other). Fig.~\ref{fig:us_experiments}(c) shows that confusion between $A3C$ and $A5C$, even common among experts, can be detected (see the result for `Expert 1') while clearly capturing the patterns of mistakes for the non-experts. In addition, Fig.~\ref{fig:us_experiments}(d) shows that our model outperforms MBEM \cite{khetan2017learning} again in classification accuracy and the quality of CM estimation. Lastly, the higher classification accuracy of our model with respect to the other baseline models illustrates again that modelling individual annotators improves robustness to label noise. 


\vspace{-2mm}
\section{Discussion and Conclusion}
\vspace{-2mm}

We introduced a new theoretically grounded algorithm for simultaneously recovering the label noise of multiple annotators and the ground truth label distribution. Our method enjoys implementation simplicity, requiring only adding a regularization term to the loss function. Experiments on both synthetic and real data sets have shown superior performance over the common EM-based methods in terms of both classification accuracy and the quality of confusion matrix estimation. Comparison against the other modern noise-robust methods demonstrates that the modelling individual annotators improves robustness to label noise. Furthermore, the method is capable of estimating annotation noise even when there is a single label per image. 

Our work was primarily motivated by medical imaging applications for which the number of classes are mostly limited to below 10. However, future work shall consider imposing structures on the confusion matrices to broaden up the applicability to massively multi-class scenarios e.g. introducing taxonomy based sparsity \cite{van2018lean} and low-rank approximation. We also assumed that there is only one ground truth for each input; this no longer holds true when the input images are truly ambiguous---recent advances in modelling multi-modality of label distributions \cite{saeedi2017multimodal,kohl2018probabilistic} potentially facilitate relaxation of such assumption. Another limiting assumption is the image independence of the annotator's label noise. The majority of disagreement between annotators arise in the difficult cases. Integrating such input dependence of label noise \cite{yan2010modeling,xiao2015learning} is also a valuable next step.



\vspace{-4mm}
\subsubsection*{Acknowledgments}
\vspace{-3mm}
We would like to thank Alon Daks, Israel Malkin and Pouya Samangouei at Butterfly Network for their feedback, and Dr. Linda Moy, MD of NYU Langone Medical Center for providing references on inter-reader variability in radiology. RT is supported by Microsoft Research Scholarship. 



\clearpage
{\small
\bibliographystyle{unsrt}
\bibliographystyle{ieee}
\bibliography{egpaper_final}
}

\newpage

\appendix
\section{Data sets, training and architectures}
\textbf{Data sets. } In this work, we verified our method on three classification datasets: MNIST digit classification dataset \cite{lecun1998gradient}; CIFAR-10 object recognition dataset \cite{krizhevsky2009learning}; the cardiac view classification (CVC) dataset from a handheld ultra-sound probe. The MNIST dataset consists of $60,000$ training and $10,000$ testing examples, all of which are $28\times28$ grayscale images of digits from $0$ to $9$. The CIFAR-10 dataset consists of $50,000$ training and $10,000$ testing examples, all of which are $32\times32$ coloured natural images drawn from $10$ classes. The CVC data set contains $26,2000$ training and $20,000$ test examples, which are grayscale images of size $96\times96$ from $6$ different cardiac views. Each image is labelled by a subset of $8$ annotators ($6$ sonographers and $2$ non-experts).  
\\ \\
\textbf{Training. } For all experiments, we employ the same training scheme unless otherwise stated. We optimize parameters using Adam \cite{kingma2014adam} with initial learning rate of $10^{-3}$ and $\beta = [0.9, 0.999]$, with minibatches of size $50$ and train for $200$ epochs. For our method, we set the scale of the trace regularization to $\lambda=0.01$. For the training of the EM-based approaches (Model-Bootstrapped EM \cite{khetan2017learning} and generalized EM \cite{raykar2009supervised}), we train the base classifier for $200$ epochs in total over the course of the EM steps, following the same protocol above. For CIFAR-10, we performed two iterations of EM algorithm ($T=2$) and $100$ epochs worth of gradient descent steps during each E-step to update the parameters of the base classifier ($G=100$ epochs), following the original implementation in \cite{khetan2017learning}. For the experiments on the CVC data set, we run more rounds of EM with $T=10$ and $G=20 $ epochs. In all cases, we hold out $10\%$ of training images as a validation set and best model is selected based on the validation accuracy over the course of training. No data augmentation is performed during training in all three data sets. We note that for CIFAR-10, we, in addition, decreased the learning rate by a factor of $10$ at every multiple of $50$ in a similar fashion to the schedule used in \cite{Huang2017LearningDR,he2016deep,springenberg2014striving}.
\\ \\
\textbf{Architectures. } For MNIST, the base classifier was defined as a CNN architecture comprised of $4$ convolution layers, each with $3\times3$ kernels follower by Relu. The number of kerners in respective layers are $\{32, 32, 64, 64\}$. After the first two convolution layers, we perform $2\times2$ max-pooling, and after the last one, we further down-sample the features with Global Average Pooling (GAP) prior to the final fully connected layer. For the CVC dataset, we employed the same architecture, but with increased number of kernels i.e. $\{128, 128, 128, 128\}$. For CIFAR-10, we used a 50-layer ResNet \cite{he2016deep}.

\section{Confusion matrices of \textit{pairwise-flippers} and \textit{hammer-spammers}}
For MNIST experiments, we considered two different models of annotator types: (i) \textit{pairwise-flipper} and (ii) \textit{hammer-spammer}. Example confusion matrices for both cases are shown in Fig.~\ref{fig:annotator_groups}.
For each annotator type and skill level $p$, we create a group of $5$ annotators by generating confusion matrices (CMs) from the associated distribution. More specifically, each CM is generated by perturbing the mean skill level $p$ by injecting a small Gaussian noise $\epsilon \sim \text{Normal}(0, 0.01)$ and choosing the flipping target class randomly in the case of a \textit{pairwise-flipper}. 


\begin{figure}[ht]
\center
    \begin{subfigure}[]{1.0\linewidth}
        \caption{Pairwise-flippers with mean skill-level $p=0.3$}
		\includegraphics[width=\linewidth]{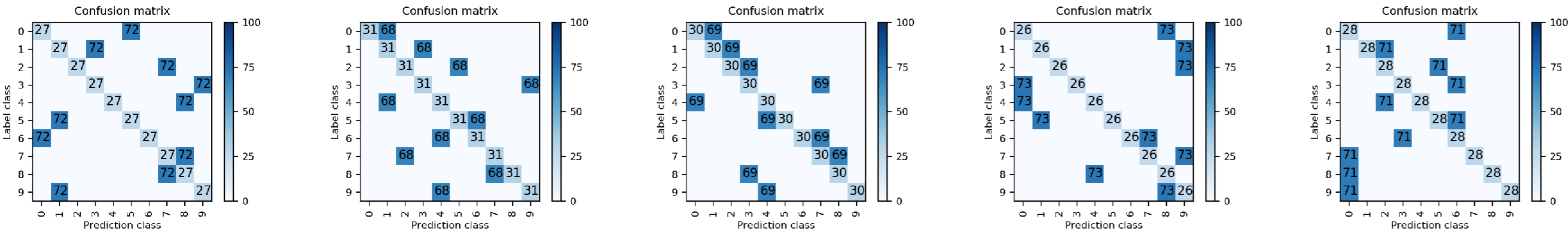}
	\end{subfigure}
	\begin{subfigure}[]{1.0\linewidth}
		\caption{Spammer-hammers skill-level $p=0.5$}
		\includegraphics[width=\linewidth]{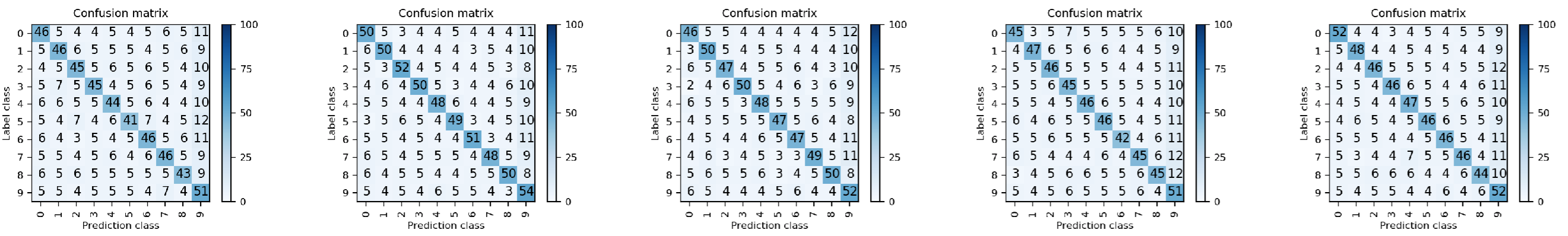}
	\end{subfigure}
	\vspace{-2mm}
	\caption{Examples of annotator groups. The value of diagonal entries are fixed constant for each annotator and is drawn from $\text{Normal}(p, 10^{-2})$.}
	\label{fig:annotator_groups}
	
\end{figure}

\section{Additional experiments on MNIST}
We present results of experiments on MNIST where models are trained on noisy labels from groups of 5 ``hammer-spammers" for a range of mean skil level $p$. Fig.~\ref{fig:random_flips} shows a comparison of our method against other EM-based approaches, while Fig.~\ref{fig:accuracy_vs_noise_levels} compares our method against other noise-robust methods without explicit modelling of individual annotators. Our method consistently achieves comparable or better accuracy with respect to the baselines. 

\begin{figure}[ht]
	\center
	\vspace{-2mm}
	\hfill
	\begin{subfigure}[]{1.0\linewidth}
		\includegraphics[width=\linewidth]{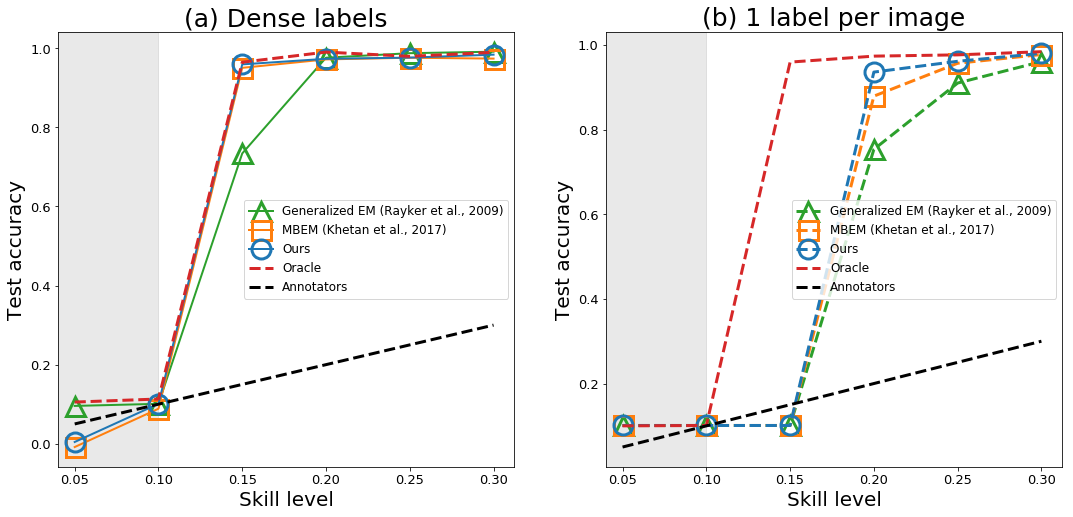}
	\end{subfigure}
	\begin{subfigure}[]{1.0\linewidth}
		\includegraphics[width=\linewidth]{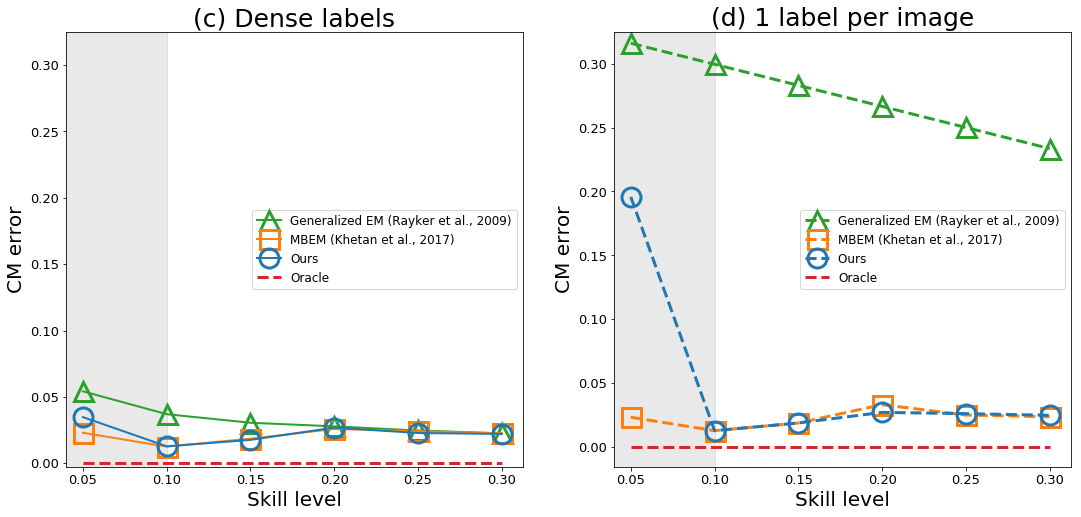}
	\end{subfigure}
	\vspace{-2mm}
	\caption{Comparison between our method, generalized EM, MBEM trained on noisy labels on MNIST from ``hammer-spammers" for a range of mean skill level $p$. (a), (b) show classification accuracy in two cases, one where all annotators label each example and the other where only one label is available per example. (c), (d) quantify the CM recovery error. The shaded areas represent the cases where the average CM over the annotators are not diagonally dominant. }
	\label{fig:random_flips}
	\vspace{-2mm}
\end{figure}

\begin{figure}[ht]
	\begin{center}
		\includegraphics[width=1.0\linewidth]{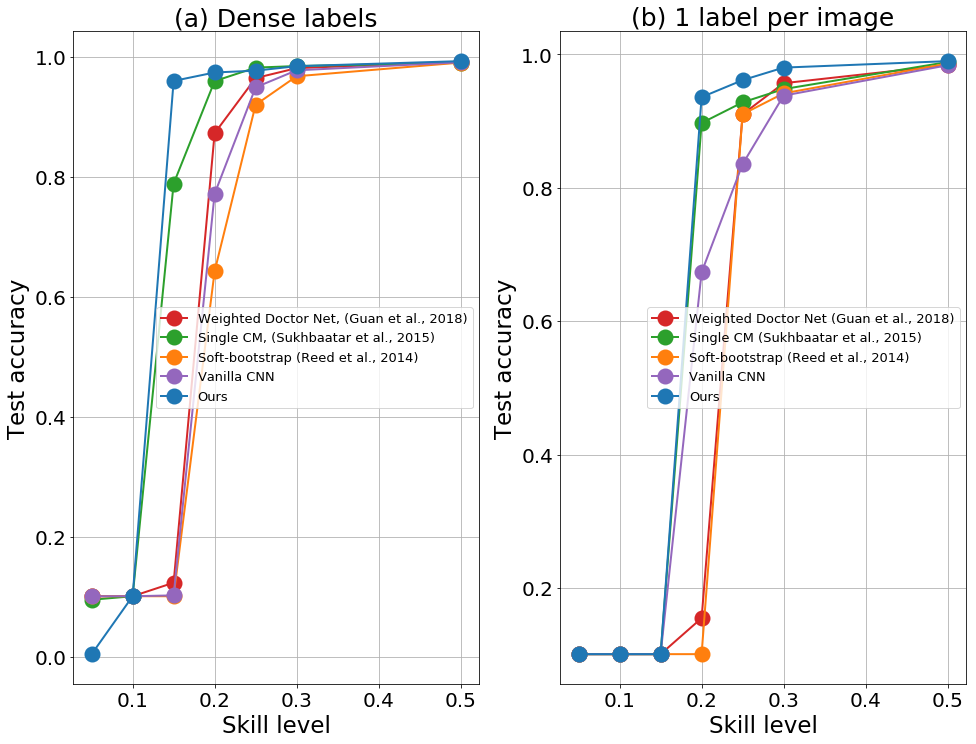}
	\end{center}
    \caption{Classification accuracy on MNIST of different noise-robust models as a function of the mean annotator skill level $p$ in two cases. Here, for each mean skill-level $p$, a group of $5$ ``pairwise flippers" is formed and used to generate labels. (a). each example receives labels from all the annotators. (b). each example is labelled by only 1 randomly selected annotator. }
	\label{fig:accuracy_vs_noise_levels}
	\vspace{-5mm}
\end{figure}

\section{Ablation study on trace regularization on MNIST}
We compare our method on MNIST against the case where the trace norm regularization is removed (results on CIFAR-10 and CVC datasets are given in the main text). Fig.~\ref{fig:ablation_trace} shows that adding the trace norm generally improves the performance in terms of both classification accuracy and CM estimation error, and such improvement is pronounced in the presence of larger noise i.e. lower skill levels of annotators. We also observe that when the noise level is low, our model still attains very high accuracy even without trace norm regularization. This can be explained by the natural robustness of the CNN classifier; if the amount of label noise is sufficiently small, the base classifier is still capable of learning the true label distribution well. This, in turn, allows the model to separate annotation noise from true label distribution, improving the quality of CM estimation and thus the overall performance. However, in the presence of large label noise, having trace-norm regularization shows evident benefits. 

\begin{figure}[ht]
	\center
	\hfill
	\begin{subfigure}[]{1.0\linewidth}
		\includegraphics[width=\linewidth]{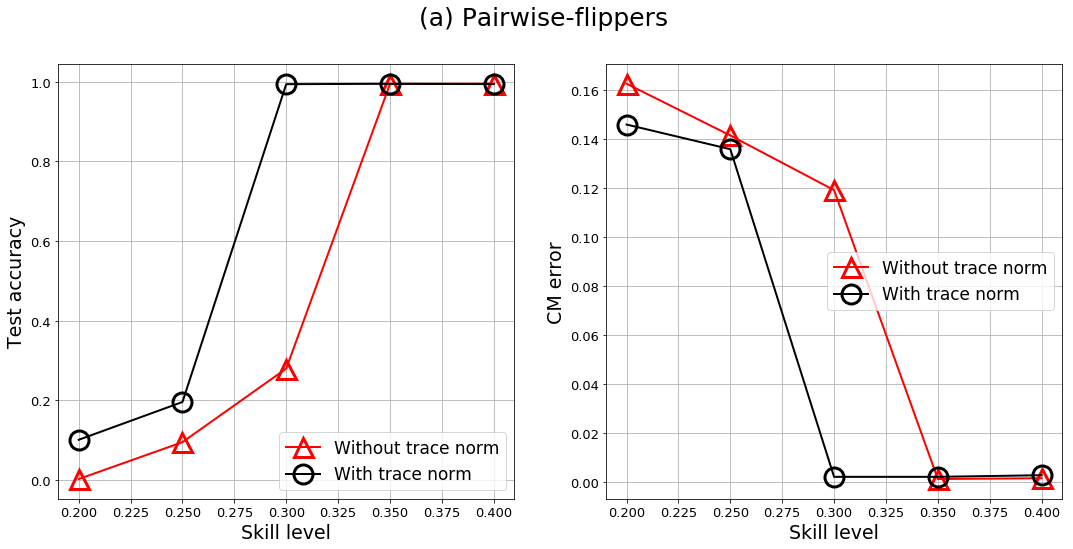}
	\end{subfigure}
	\hfill
	\begin{subfigure}[]{1.0\linewidth}
		\includegraphics[width=\linewidth]{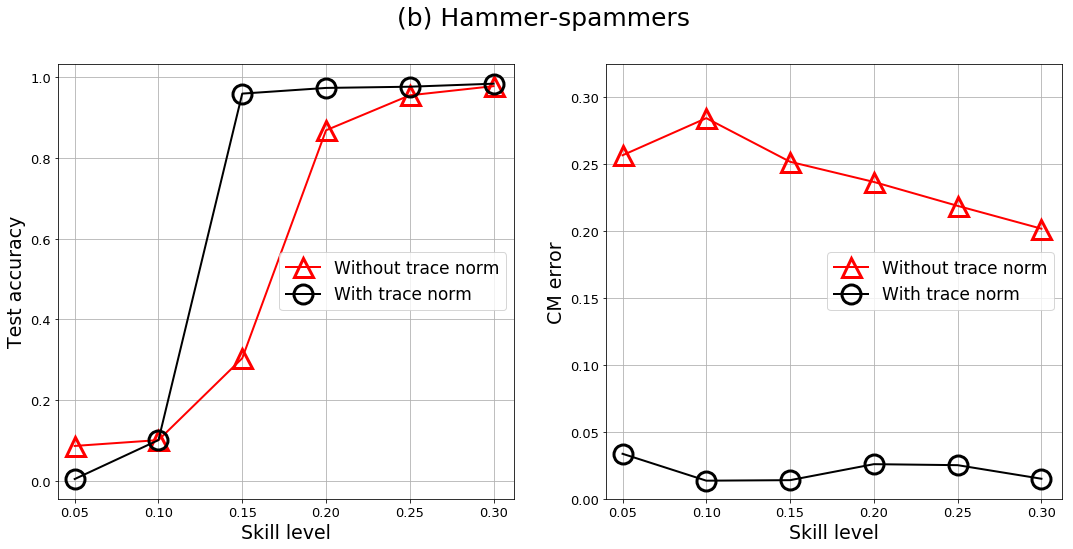}
	\end{subfigure}

	\caption{Comparison between our method with and without trace norm  on MNIST. Results for two annotator groups, consisting of ``hammer-spammers" and "pairwise-flippers" are shown for a range of mean skill level $p$.  }
	\label{fig:ablation_trace}
\end{figure}

\section{Pseudo-codes of our method, generalized EM and MBEM}
Here we provide pseudo-codes of our method (Algorithm 1), generalized EM \cite{raykar2009supervised} (Algorithm 2) and model-bootstrapped EM \cite{khetan2017learning} (Algorithm 3) to clarify the differences between different methods for jointly learning the true label distribution and confusion matrices of annotators in eq.~2 in the main text. Given the training set $\mathcal{D} = \{\textbf{x}_n, \tilde{y}^{(1)}_{n},...,\tilde{y}^{(R)}_{n}\}_{n=1}^{N}$, each example may not be labelled by all the annotators. In such cases, for ease of notation, we assign pseudo class $\tilde{y}^{(r)}_{n}=-1$ to fill the missing labels. The comparison between these three algorithms illustrates the implementational simplicity of our method, despite the comparable or superior performance demonstrated on all three datasets. 

\begin{algorithm*}
	\caption{Our method}
	\label{alg:ourmethod}
	\footnotesize
	\begin{algorithmic}
		\State \textbf{Inputs}: $\mathcal{D} = \{\textbf{x}_n, \tilde{y}^{(1)}_{n},...,\tilde{y}^{(R)}_{n}\}_{n=1}^{N}$,
$\lambda:$ scale of trace regularizer
		\State \textbf{Initialize the confusion matrices $\{\hat{\mathbf{A}}^{(r)}\}_{r=1}^{R}$ to identity matrices}
		\State \textbf{Initialize the parameters of the base classifier $\theta$}
		\State \textbf{Learn $\theta$ and $\{\hat{\mathbf{A}}^{(r)}\}_{r=1}^{R}$ by performing minibatch SGD on the combined loss}:
		\vspace{-2mm}
		$$\theta, \{\hat{\mathbf{A}}^{(r)}\}_{r=1}^{R} \longleftarrow \text{argmin}_{\theta, \{\hat{\mathbf{A}}^{(r)}\}} \Big{[}\sum_{i=1}^{N}\sum_{r=1}^{R} \mathbbm{1}(\tilde{y}_{i}^{(r)} \neq -1)\cdot \text{CE}(\hat{\mathbf{A}}^{(r)}\hat{\textbf{p}}_{\theta}(\mathbf{x}_i), \tilde{y}^{(r)}_{i})  + \lambda \sum_{r=1}^{R}\text{tr}(\hat{\textbf{A}}^{(r)}) \Big{]}$$ 
	    \vspace{-3mm}
		\State \textbf{Return:   } $\hat{\mathbf{p}}_{\theta}$ and $\{\hat{\mathbf{A}}^{(r)}\}_{r=1}^R$
	\end{algorithmic}
\end{algorithm*}
\vspace{-5mm}
\begin{algorithm*}
	\caption{Generalized EM \cite{raykar2009supervised}}
	\label{alg:generalised_em}
	\footnotesize
	\begin{algorithmic}
		\State \textbf{Inputs}: $\mathcal{D} = \{\textbf{x}_n, \tilde{y}^{(1)}_{n},...,\tilde{y}^{(R)}_{n}\}_{n=1}^{N}$,
$T:$ \# EM steps, $G:$ \# SGD in each M-step
		
		\State \textbf{Initialize posterior distribution by the mean labels}: for $j=1,..., L, n=1,...,N$
		\vspace{-2mm}
			$$q_{nj} ^{(0)}: = p(y_{n}=j|\textbf{x}_{n}, \{\tilde{y}^{(r)}_{n}\}_r, \theta^{(0)}) \longleftarrow R^{-1}\sum_{r=1}^R \mathbbm{1}(\tilde{y}_{i}^{(r)}=j) $$
		\vspace{-3mm}
		
		\State \textbf{Initialize the parameters of the base classifier $\theta$}
		\State \textbf{Repeat $T$ times:}
		\State $\,\,\,\,\,\,\,\,$ \textbf{M-step for $\theta$}. Learn the base classifier  $\hat{\mathbf{p}}_{\theta}$ by performing minibatch SGD for $G$ iterations 
		\vspace{-2mm}
		$$\theta^{(t+1)} \longleftarrow \text{argmin}_{\theta}\Big{[}-\sum_{n=1}^{N}\sum_{l=1}^{L}q_{nj} ^{(t)}\cdot \text{log } p(y_{n}=l|\mathbf{x}_{n}, \theta) \Big{]}$$ 
	    \vspace{-3mm}
	    
		\State $\,\,\,\,\,\,\,\,$ \textbf{M-step for $\{\hat{\mathbf{A}}^{(r)}\}_{r=1}^{R}$}. Estimate the confusion matrices 
		\vspace{-2mm}
		$$\hat{a}_{ji}^{(r),t+1} \longleftarrow \frac{\sum_{n=1}^{N}\mathbbm{1}(\tilde{y}_{i}^{(r)} \neq -1)\cdot \mathbbm{1}(\tilde{y}^{(r)}_{n}=i)\cdot q_{nj} ^{(t)}}{\sum_{n=1}^{N} \mathbbm{1}(\tilde{y}_{i}^{(r)} \neq -1)\cdot q_{nj} ^{(t)}}$$ 
		\vspace{-3mm}
		\State $\,\,\,\,\,\,\,\,$  \textbf{E-step}. Estimate the posterior label distribution
		\vspace{-2mm}
		$$q_{nj} ^{(t+1)} \longleftarrow \frac{ p(y_{n}=j|\mathbf{x}_{n}, \theta^{(t+1)})\cdot\prod_{r=1}^{R}\big{(}\hat{a}^{(r),t+1}_{j\tilde{y}^{(r)}_{n}}\big{)}^{\mathbbm{1}(\tilde{y}_{i}^{(r)} \neq -1)}}{\sum_{l=1}^L p(y_{n}=l|\mathbf{x}_{n}, \theta^{(t+1)})\cdot\prod_{r=1}^{R}\big{(}\hat{a}^{(r),t+1}_{l\tilde{y}^{(r)}_{n}}\big{)}^{\mathbbm{1}(\tilde{y}_{i}^{(r)} \neq -1)}}$$
        \vspace{-3mm}
		\State \textbf{Return:   } $\hat{\mathbf{p}}_{\theta^{(T)}}$ and $\{\hat{\mathbf{A}}^{(r),T}\}_{r=1}^R$
	\end{algorithmic}
\end{algorithm*}
\vspace{-5mm}
\begin{algorithm*}
	\caption{Model-Bootstrapped EM \cite{khetan2017learning}}
	\label{alg:generalised_em}
	\footnotesize
	\begin{algorithmic}
		\State \textbf{Inputs}: $\mathcal{D} = \{\textbf{x}_n, \tilde{y}^{(1)}_{n},...,\tilde{y}^{(R)}_{n}\}_{n=1}^{N}$,
$T:$ \# EM steps, $G:$ \# SGD in each M-step
		
		\State \textbf{Initialize posterior distribution by the mean labels}: for $j=1,..., L, n=1,...,N$
		\vspace{-2mm}
			$$q_{nj} ^{(0)}: = p(y_{n}=j|\textbf{x}_{n}, \{\tilde{y}^{(r)}_{n}\}_r, \theta^{(0)}) \longleftarrow R^{-1}\sum_{r=1}^R \mathbbm{1}(\tilde{y}_{i}^{(r)}=j) $$
		\vspace{-3mm}
		\State \textbf{Initialize the parameters of the base classifier $\theta$}
		\State \textbf{Repeat $T$ times:}
		\State $\,\,\,\,\,\,\,\,$ \textbf{M-step for $\theta$}. Learn the base classifier  $\hat{\mathbf{p}}_{\theta}$ by performing minibatch SGD for $G$ iterations 
		\vspace{-2mm}
		$$\theta^{(t+1)} \longleftarrow \text{argmin}_{\theta}\Big{[}-\sum_{n=1}^{N}\sum_{l=1}^{L}q_{nj} ^{(t)}\cdot \text{log } p(y_{n}=l|\mathbf{x}_{n}, \theta)\Big{]}$$ 
	    \vspace{-3mm}
	    \State $\,\,\,\,\,\,\,\,$ \textbf{Predict on training examples.} for $n=1,...,N$:
	    \vspace{-2mm}
	    $$c_{n} \longleftarrow \text{argmax}_{l\in\{1,...,L\}}\text{ }p(y_{n}=l|\mathbf{x}_{n}, \theta^{(t+1)}) $$
	    \vspace{-3mm}
		\State $\,\,\,\,\,\,\,\,$ \textbf{M-step for $\{\hat{\mathbf{A}}^{(r)}\}_{r=1}^{R}$}. Estimate the confusion matrices. For $i,j=1,...,L$ and $r=1,...,R$:
		\vspace{-2mm}
		$$\hat{a}_{ji}^{(r),t+1} \longleftarrow \frac{\sum_{n=1}^{N}\mathbbm{1}(\tilde{y}_{i}^{(r)} \neq -1)\cdot \mathbbm{1}(\tilde{y}^{(r)}_{n}=i)\cdot \mathbbm{1}(c_{n}=j)}{\sum_{n=1}^{N} \mathbbm{1}(\tilde{y}_{i}^{(r)} \neq -1)\cdot \mathbbm{1}(c_{n}=j)}$$ 
		\vspace{-3mm}
		 \State $\,\,\,\,\,\,\,\,$ \textbf{Update prior label distribution.} for $l=1,...,L$:
	    \vspace{-3mm}
	    $$p_{l} \longleftarrow N^{-1}\sum_{n=1}^N \mathbbm{1}(c_{n}=l) $$
	    \vspace{-3mm}
		\State $\,\,\,\,\,\,\,\,$  \textbf{E-step}. Estimate the posterior label distribution
		\vspace{-2mm}
		$$q_{nj} ^{(t+1)} \longleftarrow \frac{p_{j}\cdot\prod_{r=1}^{R}\big{(}\hat{a}^{(r),t+1}_{j\tilde{y}^{(r)}_{n}}\big{)}^{\mathbbm{1}(\tilde{y}_{i}^{(r)} \neq -1)}}{\sum_{l=1}^L p_{l}\cdot\prod_{r=1}^{R}\big{(}\hat{a}^{(r),t+1}_{l\tilde{y}^{(r)}_{n}}\big{)}^{\mathbbm{1}(\tilde{y}_{i}^{(r)} \neq -1)}}$$
        \vspace{-3mm}
		\State \textbf{Return:   } $\hat{\mathbf{p}}_{\theta^{(T)}}$ and $\{\hat{\mathbf{A}}^{(r),T}\}_{r=1}^R$
	\end{algorithmic}
\end{algorithm*}
\clearpage
\newpage
\onecolumn
\section{TensorFlow codes}
\subsection{Probabilistic model and loss function}

\begin{lstlisting}[language=Python, caption={Implementation of the probabilistic model and the proposed loss function given in eq.~\eqref{eq:objective_sparse}. The final loss is minimized to learn jointly the confusion matrices of the respective annotators and the parameters of the classifier. The details of the used functions are given in Sec.~\ref{sec:code_conf_mat} \&~\ref{sec:code_loss}} , label={lst:rectified relu}]
import tensorflow as tf

# inputs and annotators' labels
images = tf.placeholder(tf.float32, (None, image_shape))
labels = tf.placeholder(tf.int32,(None, num_annotators, num_classes))

# classifier for estimating true label distribution
logits = classifier(images)

# confusion matrices of annotators
confusion_matrices = confusion_matrix_estimators(num_annotators, num_classes)

# loss function
# 1. weighted cross-entropy
weighted_cross_entropy = cross_entropy_over_annotators(labels, logits, confusion_matrices)

# 2. trace of confusion matrices:
trace_norm = tf.reduce_mean(tf.trace(confusion_matrices))

# final loss (eq.(4))
total_loss = weighted_cross_entropy + scale * trace_norm
\end{lstlisting}

\subsection{Defining confusion matrices of annotators}\label{sec:code_conf_mat}
\begin{lstlisting}[language=Python, caption={The confusion matrices are defined as \texttt{tf.Variable}. The positivity of the elements of confusion matrices is ensured by passing them through a soft-plus function. }, label={lst:rectified relu}]
import numpy as np
import tensorflow as tf

def confusion_matrix_estimators(num_annotators, num_classes):
    """Defines confusion matrix estimators.
    This function defines a set of confusion matrices that characterize respective annotators.
    Here (i, j)th element in the annotator confusion matrix of annotator a is given by 
    P(label_annotator_a = j| label_true = i) i.e. the probability that the annotator assigns label j to the image when the ground truth label is i.

    Args:
        num_annotators: Number of annotators
        num_classes: Number of classes.

    Returns:
        confusion_matrices: Annotator confusion matrices. A `Tensor` of shape of shape [num_annotators, num_classes, num_classes]
    """
    # initialise so the confusion matrices are close to identity matrices. 
    w_init = tf.constant(
        np.stack([6.0 * np.eye(num_classes) - 5.0 for j in range(num_annotators)]),
        dtype=tf.float32,
    )
    rho = tf.Variable(w_init, name='rho')

    # ensure positivity
    rho = tf.nn.softplus(rho)
    
    # ensure each row sums to one
    confusion_matrices = tf.divide(rho, tf.reduce_sum(rho, axis=-1, keepdims=True))
    return confusion_matrices
\end{lstlisting}

\newpage
\subsection{Cross-entropy loss with sparse and noisy labels}\label{sec:code_loss}
\begin{lstlisting}[language=Python, caption={Implementation of the cross entropy loss function. Here we use zero vectors as the ``one-hot'' representations of missing labels.}, label={lst:rectified relu}]
import numpy as np
import tensorflow as tf

def cross_entropy_over_annotators(labels, logits, confusion_matrices):
    """ Cross entropy between noisy labels from multiple annotators and their confusion matrix models.
    Args:
        labels: One-hot representation of labels from multiple annotators. 
            tf.Tensor of size [batch, num_annotators, num_classes]. Missing labels are assumed to be
            represented as zero vectors.
        logits: Logits from the classifier. tf.Tensor of size [batch, num_classes]
        confusion_matrices: Confusion matrices of annotators. tf.Tensor of size 
            [num annotators, num_classes, num_classes]. The (i, j) th element of the confusion matrix 
            for annotator a denotes the probability P(label_annotator_a = j|label_true = i).
    Returns:
         The average cross-entropy across annotators and image examples.
    """
    # Treat one-hot labels as probability vectors
    labels = tf.cast(labels, dtype=tf.float32)

    # Sequentially compute the loss for each annotator
    losses_all_annotators = []
    for idx, labels_annotator in enumerate(tf.unstack(labels, axis=1)):
        loss = sparse_confusion_matrix_softmax_cross_entropy(
            labels=labels_annotator,
            logits=logits,
            confusion_matrix=confusion_matrices[idx, :, :],
        )
        losses_all_annotators.append(loss)

    # Stack them into a tensor of size (batch, num_annotators)
    losses_all_annotators = tf.stack(losses_all_annotators, axis=1)

    # Filter out annotator networks with no labels. This allows you train
    # annotator networks only when the labels are available.
    has_labels = tf.reduce_sum(labels, axis=2)
    losses_all_annotators = losses_all_annotators * has_labels
    return tf.reduce_mean(tf.reduce_sum(losses_all_annotators, axis=1))

def sparse_confusion_matrix_softmax_cross_entropy(labels, logits, confusion_matrix):
    """Cross entropy between noisy labels and confusion matrix based model for a single annotator.
    Args:
        labels: One-hot representation of labels. Tensor of size [batch, num_classes].
        logits: Logits from the classifier. Tensor of size [batch, num_classes]
        confusion_matrix: Confusion matrix of the annotator. Tensor of size [num_classes, num_classes].
    Returns:
         The average cross-entropy across annotators for image examples
         Returns a `Tensor` of size [batch_size].
    """
    # get the predicted label distribution
    preds_true = tf.nn.softmax(logits)

    # Map label distribution into annotator label distribution by
    # multiplying it by its confusion matrix.
    preds_annotator = tf.matmul(preds_true, confusion_matrix)
    
    # cross entropy
    preds_clipped = tf.clip_by_value(preds_annotator, 1e-10, 0.9999999)
    cross_entropy = tf.reduce_sum(-labels * tf.log(preds_clipped), axis=-1)
    return cross_entropy
\end{lstlisting}

\end{document}